\documentclass[nohyperref]{article}

\usepackage{microtype}
\usepackage{graphicx}
\usepackage{subfigure}
\usepackage{booktabs} 
\usepackage{url}
\usepackage{algorithm}
\usepackage{algorithmic}
\usepackage{amsmath,amsfonts,bm}
\usepackage{amsthm,amssymb}
\usepackage{bbm}   
\usepackage{diagbox}
\usepackage{booktabs}  
\usepackage{color}
\usepackage{xcolor}
\definecolor{darkblue}{RGB}{25, 50, 112}
\usepackage[colorlinks=true, linkcolor=black, citecolor=darkblue]{hyperref}
\usepackage{float}
\usepackage{lipsum}  
\usepackage{enumitem} 
\usepackage{multirow}
\usepackage{graphicx}  
\usepackage{listings}

\usepackage[accepted]{icml2022}

\usepackage{amsmath}
\usepackage{amssymb}
\usepackage{mathtools}
\usepackage{amsthm}

\usepackage[capitalize,noabbrev]{cleveref}  

\theoremstyle{plain}
\newtheorem{theorem}{Theorem}
\newtheorem{proposition}[theorem]{Proposition}
\newtheorem{lemma}[theorem]{Lemma}

\theoremstyle{definition}

\theoremstyle{remark}
\newtheorem{remark}[theorem]{Remark}

\usepackage[textsize=tiny]{todonotes}

\newcommand{\emprisk}{\widehat{R}_{n}}
\newcommand{\empadvrisk}{\widehat{R}^{\text{adv}}_{n}}
\newcommand{\advrisk}{R^{\text{adv}}}

\newcommand{\vect}[1]{\ensuremath{\mathbf{#1}}}

\newcommand{\argmin}{\mathop{\rm argmin}}
\newcommand{\argmax}{\mathop{\rm argmax}}

\newcommand{\R}{\mathbb{R}}

\newcommand{\G}{\mathcal{G}}
\newcommand{\Eover}[2]{\mathbb{E}_{#1}\left[#2\right]}

\newcommand{\x}{\vect{x}}

\newcommand{\y}{\vect{y}}
\newcommand{\z}{\vect{z}}

\renewcommand{\H}{\mathcal{H}}

\newcommand{\cX}{\mathcal{X}}
\newcommand{\cY}{\mathcal{Y}}
\newcommand{\cZ}{\mathcal{Z}}

\newcommand{\robboostalg}{\textsc{MRBoost}}
\newcommand{\robboostalgnn}{\textsc{MRBoost.NN}}
\newcommand{\samplerexp}{\textsc{Sampler.Exp}}
\newcommand{\samplerrnd}{\textsc{Sampler.Rnd}}
\newcommand{\samplermax}{\textsc{Sampler.Max}}
\newcommand{\samplerall}{\textsc{Sampler.All}}
\newcommand{\whole}{\textsc{Whole}}
\newcommand{\randinit}{\textsc{RndInit}}
\newcommand{\perinit}{\textsc{PerInit}}
\newcommand{\abernehty}{\textsc{RobBoost}}

\newcommand{\sampler}{\textsc{Sampler}}

\newcommand{\ce}{{\ell_{\operatorname{CE}}}}
\newcommand{\cemargin}{{\ell_{\operatorname{MCE}}}}
\newcommand{\cemargina}{{\ell_{\operatorname{MCE-A}}}}
\newcommand{\bce}{{\ell_{\operatorname{BCE}}}}

\newcommand{\amclf}[1]{{#1}^{\mathop{\rm am}}}
\newcommand{\amclftwo}[2]{{#1}^{\mathop{\rm am}}_{#2}}
\newcommand{\bI}{\mathbb{I}}
\newcommand{\pairwiseloss}[1]{{\rm mg}_{\rm L}\left({#1}\right)}
\newcommand{\margin}[1]{\mathop{\rm mg}\left({#1}\right)}
\newcommand{\marginrob}[1]{{\rm mg}_{\rm rob}\left({#1}\right)}
\newcommand{\notimplies}{%
  \mathrel{{\ooalign{\hidewidth$\not\phantom{=}$\hidewidth\cr$\implies$}}}}

\newcommand{\ie}{\emph{i.e.}}

\newcommand{\E}{\mathbb{E}}
\global\long\def\R{\mathbb{R}}
\def\gB{{\mathcal{B}}}

\newcommand{\KL}{D_{\mathrm{KL}}}
\newcommand{\sm}{\texttt{softmax}}
\def\rvx{{\mathbf{x}}}

\def\rvp{{\mathbf{p}}}
\def\vtheta{\boldsymbol{\theta}}
\newcommand{\indct}[1]{\mathbb{I}(#1)}

\icmltitlerunning{Building Robust Ensembles via Margin Boosting}

\begin{document}

\twocolumn[
\icmltitle{Building Robust Ensembles via Margin Boosting}




\begin{icmlauthorlist}
\icmlauthor{Dinghuai Zhang}{mila}
\icmlauthor{Hongyang Zhang}{waterloo}
\icmlauthor{Aaron Courville}{mila}
\icmlauthor{Yoshua Bengio}{mila}

\icmlauthor{Pradeep Ravikumar}{cmu}
\icmlauthor{Arun Sai Suggala}{google}
\end{icmlauthorlist}

\icmlaffiliation{mila}{Mila and Universit\'e de Montr\'eal}
\icmlaffiliation{waterloo}{University of Waterloo}
\icmlaffiliation{cmu}{
Carnegie Mellon University}
\icmlaffiliation{google}{Google Research}

\icmlcorrespondingauthor{Dinghuai Zhang}{dinghuai.zhang@mila.quebec}

\icmlkeywords{Adversarial Attacks, Robustness, Ensembles, Boosting}

\vskip 0.3in
]



\printAffiliationsAndNotice{}  

\begin{abstract}
In the context of adversarial robustness, a single model does not usually have enough power to defend against all possible adversarial attacks, and as a result, has sub-optimal robustness. Consequently, an emerging line of work  has focused on learning an ensemble of neural networks to defend against adversarial attacks. In this work, we take a principled approach towards building robust ensembles. We view this problem from the perspective of margin-boosting and develop an algorithm for learning an ensemble with maximum margin. Through extensive empirical evaluation on benchmark datasets, we show that our algorithm not only outperforms existing ensembling techniques, but also large models trained in an end-to-end fashion. An important byproduct of our work is a margin-maximizing cross-entropy (MCE) loss, which is a better alternative to the standard cross-entropy (CE) loss. Empirically, we show that replacing the CE loss in state-of-the-art adversarial training techniques with our MCE loss leads to significant performance improvement.
\end{abstract}

\section{Introduction}
\label{sec:intro}
The output of deep neural networks can be vulnerable to even a small amount of perturbation to the input~\citep{szegedy2013intriguing}. 
These perturbations, usually referred to as adversarial perturbations, can be imperceptible to humans and yet deceive even state-of-the-art models into making incorrect predictions. Owing to this vulnerability, there has been a great interest in understanding the  phenomenon of these adversarial perturbations~\citep{fawzi2018analysis, bubeck2019adversarial, ilyas2019adversarial}, and in designing techniques to defend against adversarial attacks~\citep{goodfellow2014explaining, madry2017towards,raghunathan2018certified, zhang2019theoretically}. 
However, coming up with effective defenses has turned out to be a hard problem as many of the proposed defenses were eventually circumvented by novel adversarial attacks~\citep{athalye2018obfuscated, tramer2020adaptive}. Even the well performing techniques, such as adversarial training (AT)~\citep{madry2017towards}, are unsatisfactory as they do not yet achieve high enough adversarial robustness. 

One of the reasons for the unsatisfactory performance of existing defenses is that they train a single neural network to defend against all possible adversarial attacks. Such single models are often not powerful enough to defend against all the moves of the adversary and as a result, have sub-optimal robustness. Consequently, an emerging line of work in adversarial robustness has focused on constructing ensembles of neural networks~\citep{kariyappa2019improving, verma2019error, Pinot2020RandomizationMH}. These ensembling techniques work under the hypothesis that an ensemble with a diverse collection of classifiers can be more effective at defending against adversarial attacks, and can lead to better robustness guarantees. This hypothesis was in fact theoretically proven to be true by \citet{Pinot2020RandomizationMH}. 
However, many of the existing ensemble based defenses were not successful, and were defeated by stronger attacks~\citep{tramer2020adaptive}. In this paper, we show that the recently proposed ensembling technique of \citet{Pinot2020RandomizationMH} can be defeated (Appendix~\ref{sec:pinot_defense_break}), thus adding the latter to the growing list of ``circumvented ensemble defenses''. This shows that the problem of designing ensembles that are robust to adversarial attacks is pretty much unsolved. 

In this work, we take a principled approach towards constructing robust ensembles. We view this problem from the perspective of boosting and ask the following question:
\begin{center}
\vskip -0.06in
\emph{How can we combine multiple base classifiers into a strong classifier that is robust to adversarial attacks?}
\end{center}\vskip -0.06in
Our answer to this question relies on the key machine learning notion of margin, which is known to govern the generalization performance of a model~\citep{bartlett1998sample}. In particular, we develop a margin-maximizing boosting framework that aims to find an ensemble with maximum margin via a two-player zero-sum game between the learner and the adversary, the solution to which is the desired max-margin ensemble. 
One of our key contributions 
is to provide an efficient algorithm (\robboostalg) for solving this game.  
Through extensive empirical evaluation on benchmark datasets, we show that our algorithm not only outperforms existing ensembling techniques, but also large models trained in an end-to-end fashion.  
An important byproduct of our work is a margin-maximizing cross-entropy (MCE) loss, which is a better alternative to the standard cross-entropy (CE) loss. 
Empirically, we demonstrate that replacing the CE loss in state-of-the-art adversarial training techniques with our MCE loss leads to significant
performance
improvement.
Our code is available at \href{https://github.com/zdhNarsil/margin-boosting}{https://github.com/zdhNarsil/margin-boosting}.

\textbf{Contributions.}
Here are the key contributions of our work:\vspace{-0.12in}
\begin{itemize}
  \setlength\itemsep{0in}
    \item We propose a margin-boosting framework for learning max-margin ensembles. Moreover, we prove the optimality of our framework by showing that it requires the weakest possible conditions on  base classifiers to output an ensemble with strong adversarial performance; 
    \item We derive a computationally efficient algorithm ($\robboostalgnn$) from our margin-boosting framework. Through extensive empirical evaluation, we demonstrate the effectiveness of our algorithm;
    \item Drawing from our boosting framework, we propose the MCE loss, a new variant of the CE loss, which improves the adversarial robustness of state-of-the-art defenses.  
\end{itemize}


\vspace*{-4mm}
\section{Preliminaries}
\label{sec:background}

In this section, we set up the notation and review necessary background on adversarial robustness and ensembles. A
consolidated list of notation can be found in Appendix~\ref{sec:app_notation}. 

\textbf{Notation.} Let $(\x,y) \in \cX \times \cY$ denote a feature-label pair following an unknown probability distribution $P$. In this work, we consider the multi-class classification problem where $\cY =  \{0, \dots K-1\}$, where $K$ is the number of classes, and assume $\cX \subseteq \R^d$.  
Let $S=\{(\x_i,y_i)\}_{i=1}^n$ be $n$ i.i.d samples drawn from $P$, and $P_n$ the empirical distribution of $S$.  
We let $h:\cX \to \cY$ denote a generic classifier which predicts the label of $\x$ as $h(\x)$. 
In practice, such a classifier is usually constructed by first constructing a \emph{score-based} classifier $g:\cX \to \R^K$ which assigns a confidence score to each class, and then mapping its output to an appropriate class using the $\argmax$ operation: $\argmax_{j\in \cY} [g(\x)]_j$.\footnote{If there are multiple optimizers to this problem, we pick one of the optimizers uniformly at random.} 
We denote the resulting classifier by $\amclf{g}$.

\textbf{Standard Classification Risk.} 
The expected classification risk of a score-based classifier $g$ is defined as $ \Eover{(X,Y)\sim P}{\ell_{0-1}(g(X), Y)}$, 
where $\ell_{0-1}(g(X), Y) = 0$ if  $\amclf{g}(X) = Y$, and $1$ otherwise. 
Since optimizing $0/1$ risk is computationally intractable, it is often replaced with convex surrogates, which we denote by $\ell(g(X),Y)$.
A popular choice for $\ell$ is the cross-entropy loss: $\ce(g(\x), y)\coloneqq -[g(\x)]_{y} + \log \left(\sum_{j \in \cY}\exp{[g(\x)]_j} \right)$. 
The population and empirical risks of classifier $g$ w.r.t loss $\ell$ are  defined as
$
R(g; \ell) \coloneqq \Eover{(X,Y) \sim P}{\ell(g(X),Y)},$ \mbox{$\emprisk(g; \ell) \coloneqq \Eover{(X,Y)\sim P_n}{\ell(g(X), Y)}.$} \vspace{0.05in}\\
\textbf{Adversarial Risk.} We consider the following robustness setting in this work: given a classifier, there is an adversary which corrupts the inputs to the classifier with the intention of making the model misclassify the inputs.  Our goal is
to design models that are robust to such adversaries. Let $\gB(\epsilon)$ be the set of perturbations that the adversary is allowed to add to the input. Popular choices for $\gB(\epsilon)$ for instance include $\ell_p$ norm balls $\{\omega: \|\omega\|_p \leq \epsilon\}$ for $p \in \{2,\infty\}$. In this work, we assume $\gB(\epsilon)$ is a compact set (this is satisfied by $\ell_p$ norm balls). Given this setting, the population and empirical  adversarial risks of a classifier $g$ w.r.t loss $\ell$ are defined as
\vspace{-2mm}
\begin{align*}
    &\advrisk(g; \ell) \coloneqq \Eover{(X,Y) \sim P}{\max_{\delta\in \gB(\epsilon)}\ell(g(X+\delta),Y)},\\ 
    &\empadvrisk(g; \ell) \coloneqq \Eover{(X,Y)\sim P_n}{\max_{\delta\in \gB(\epsilon)}\ell(g(X+\delta), Y)}.
\end{align*}
 \vspace{-4mm}
 
\textbf{Ensembles.} 
Ensembling is a popular technique in machine learning for constructing models with good generalization performance (\ie, small population risk). 
Ensembling typically involves linearly combining the predictions of several base classifiers.  Let $\H$ be the set of base classifiers, where each $h\in \H$ maps from $\cX$ to $\cY$.
In ensembling, we place a probability distribution $Q$ over $\H$ which specifies the weight of each base classifier. 
This distribution defines a score-based classifier $h_{Q}: \cX \to \R^K$ with $[h_Q(\x)]_{j} = \Eover{h\sim Q}{\bI(h(\x) = j)}$. 
This can be converted into a standard classifier using the $\argmax$ operation: $\amclftwo{h}{Q}(\x) = \argmax_{j \in \cY} [h_Q(\x)]_{j}.$   
If we have score-based base classifiers $\G = \{g_1, g_2, \dots\}$, where each $g\in \G$ maps $\cX$ to $\R^K$, we can simply linearly combine them via a set of real-valued weights $W$ over $\G$, and define a score-based classifier $g_{W}$ as $[g_W(\x)]_j = \sum_{g \in \G} W(g)[g(\x)]_j$, where $W(g) \in \R$ is the weight of the base classifier $g$. $g_W$ can be converted into a standard classifier using the $\argmax$ operation: $\amclftwo{g}{W}(\x) = \argmax_{j \in \cY} [g_W(\x)]_{j}.$ 

\textbf{Boosting.} 
Boosting is perhaps the most popular technique for creating ensembles. 
Boosting aims to address the following question: 
\textit{``Given a set of base classifiers, how can we combine them to produce an ensemble with the best predictive performance?''} 
Numerous techniques have been proposed to answer this question, with the most popular ones being \emph{margin}-boosting~\citep{freund1996experiments, breiman1999prediction, ratsch2005efficient} and \emph{greedy}-boosting~\cite{mason2000boosting, friedman2001greedy}. 
The technique developed in this work falls in the category of margin-boosting. 
Margin-boosting works under the hypothesis that a large-margin classifier has good generalization performance~\citep{bartlett1998boosting, bartlett1998sample}. 
Consequently, it aims to learn an ensemble with large margin. 
Let $\H$ be the set of base classifiers, where each $h\in \H$ maps $\cX$ to $\cY$. 
The margin of the ensemble $h_Q$, for some probability distribution $Q$ over $\H$, at point $(\x,y)$, is defined as
\vspace{-1mm}
\begin{align}
\label{eqn:margin}
\margin{Q, \x,y} \coloneqq [h_Q(\x)]_y - \max_{y' \neq y} [h_Q(\x)]_{y'}.
\end{align}
Intuitively, margin captures the confidence with which $h_{Q}$ assigns $\x$ to class $y$. 
We ideally want $\margin{Q, \x, y}$ to be large for all $(\x,y) \in S$. 
To capture this, we introduce the notion of \emph{minimum margin}  over the dataset $S$ which is defined as \mbox{$\margin{Q, S} \coloneqq \min_{(\x,y)\in S}\margin{Q, \x, y}$.} 
In margin-boosting, we aim to find a $Q$ with the largest possible minimum margin. 
This leads us to the following optimization problem: $\max_{Q \in \Delta(\H)}\margin{Q, S}$, where $\Delta(\H)$ is the set of all probability distributions over $\H$. 
AdaBoost.MR, a popular boosting algorithm for multi-class classification, can be viewed as solving this optimization problem~\citep{schapire1999improved, mukherjee2013theory}.



\vspace{-2mm}
\subsection{Related Work}
\label{sec:related_work}

\vspace{-1mm}
\textbf{Adversarial Robustness.} 
Numerous techniques have been proposed to defend neural networks against adversarial attacks \citep{Szegedy2014IntriguingPO}. These techniques broadly fall into two categories. The first category of techniques called \emph{empirical defenses} rely on heuristics and do not provide any guarantee on the robustness of the learned models. 
Adversarial training (AT)~\citep{madry2017towards} is by far the most popular defense in this category. 
AT has been successfully improved by many recent works such as \citet{zhang2019theoretically, Zhang2019YouOP, wang2019improving, Carmon2019UnlabeledDI, Zhang2020AttacksWD, Shi2020InformativeDF}. 
The second category of techniques 
called \emph{certified defenses} output models whose robustness can be certified in the following sense: 
at any given point, they can provide a certificate  proving the robustness of the learned model to adversarial attacks at that point.
Early computationally efficient approaches in this category were developed by \citet{raghunathan2018certified, wong2018provable}.
Several recent techniques such as randomized smoothing~\citep{cohen2019certified, salman2019provably, Zhang2020BlackBoxCW,blum2020random, yang2021certified} have improved upon these early works and even scale to ImageNet size datasets.

Several recent works have attempted to use ensembles to defend against adversarial attacks.
\citet{sen2019empir} train the base classifiers of the ensemble independent of each other,
which ignores the interaction between base classifiers~\citep{pang2019improving}. 
Other works \citep{verma2019error, kariyappa2019improving, pang2019improving, meng2020athena} simultaneously learn all the components of the ensemble, which requires huge memory and computational resources.
These works add diversity promoting regularizers to their training objectives to learn good ensembles. 
A number of these defences are known to be rather weak~\citep{tramer2020adaptive}. 
There are also works which use boosting inspired approaches and construct ensembles in a sequential manner~\citep{Pinot2020RandomizationMH, abernethy2021multiclass}.
Many of these techniques rely on heuristics and are rather weak. 
In Appendix~\ref{sec:pinot_defense_break}, we empirically show that the defense of \citet{Pinot2020RandomizationMH} can be circumvented with carefully designed adversarial attacks. 
Moreover, we show that our boosting algorithm has better performance than the algorithm of~\citet{abernethy2021multiclass}.

\vspace{-1mm}
\textbf{Boosting.} Boosting has a rich history in both computer science (CS) and statistics. The CS community takes a game-theoretic perspective of boosting and views boosting algorithms as playing a game against a base/weak learner~\citep{freund1995desicion}. The statistical community views it as greedy stagewise optimization~\citep{friedman2001greedy, mason2000boosting}. Both these views have contributed to the development of popular boosting techniques such as AdaBoost~\citep{freund1995desicion}, XGBoost~\citep{chen2016xgboost}. 
In this work, we take the game-theoretic perspective to build robust ensembles. 
Recently, boosting has seen a revival in the deep learning community. This is because boosting techniques consume less memory than end-to-end training of deep networks and can accommodate much larger models in limited memory~\citep{huang2017learning, nitanda2018functional, suggala2020generalized}. Moreover, boosting techniques are easier to understand from a theoretical and optimization standpoint and can make neural networks easy to adopt in critical applications.

\vspace{-2mm}
\section{Margin-Boosting for Robustness}
\label{sec:margin_boosting}
In this section, we present our margin-boosting framework for building robust ensembles. 
Let $\H$ be a compact set of base classifiers. Typical choices for $\H$ include the set of all decision trees of certain depth, and the set of all neural networks of bounded depth and width. Given $\H$, our goal is to design an ensemble (\emph{i.e.,} identify a $Q \in \Delta(\H)$) which has the best possible robustness towards adversarial attacks.  The starting point for our boosting framework is the observation that large margin classifiers tend to generalize well to unseen data. In fact, large margin has been attributed to the success of popular ML techniques such as SVMs and boosting~\citep{bartlett1998sample, bartlett1998boosting, mason2000improved}.  So, in this work, we learn ensembles with large margins to defend against adversarial attacks. 
However, unlike ordinary boosting, ensuring large margins for data points in $S$ does not suffice for adversarial robustness. In robust boosting, we need \emph{large margins even at the perturbed points} for the ensemble to have good adversarial generalization~\citep{khim2018adversarial, yin2019rademacher}. 
Letting $h_Q$ be our ensemble, we want $\margin{Q, \x+\delta, y}$ to be large for all $(\x,y) \in S$, $\delta \in \gB(\epsilon)$.
To capture this, we again introduce the notion of \emph{minimum robust margin} of $h_Q$ which is defined as \mbox{$\marginrob{Q, S} \coloneqq \min_{(\x,y)\in S}\min_{\delta \in \gB(\epsilon)}\margin{Q, \x+\delta, y}$}. 
In our boosting framework, we aim to find a $Q$ with the largest possible  $\marginrob{Q, S}$, which leads us to the following optimization problem: $\max_{Q\in \Delta(\H)} \marginrob{Q, S}$. This problem can be equivalently written as 
\begin{align}
\label{eqn:robust_boosting_game}
\max_{Q \in \Delta(\H)}\min_{\substack{(\x,y) \in S,\\y' \in \cY\setminus \{y\}, \delta \in \gB(\epsilon)}} [h_Q(\x+\delta)]_y - [h_Q(\x+\delta)]_{y'}.
\end{align}
In this work, we often refer to such max-min problems as two-player zero-sum games.
\vspace{-2mm}
\subsection{Margin-Boosting is Optimal} 
\label{sec:wl_condition}
Before we present our algorithm for solving the max-min problem in Equation~\eqref{eqn:robust_boosting_game}, we try to understand the margin-boosting framework from a theoretical perspective. In particular, we are interested in studying the following questions pertaining to the quality of our boosting framework:
\vspace{-0.1in}
\begin{enumerate}
    \item Under what conditions on $\H$ does the boosting framework output an ensemble with $100\%$ adversarial accuracy on the training set?
    \vspace{-0.05in}
    \item Are these conditions  on $\H$ optimal? Can there be a different boosting framework that outputs an $100\%$ accurate ensemble under milder conditions on $\H$? 
\end{enumerate}
\vspace{-0.1in}

Understanding these questions can aid us in designing appropriate base hypothesis classes $\H$ for our boosting framework. 
The choice of $\H$ is crucial as it can significantly impact learning and generalization. If $\H$ is too weak to satisfy the required conditions, then the boosting framework cannot learn a robust ensemble. 
On the other hand, using a more complex $\H$ than necessary can result in overfitting. 
Understanding the second question can thus help us compare various boosting frameworks. 
For instance, consider two boosting frameworks $B_1$ and $B_2$. 
If $B_1$ requires more powerful hypothesis class $\H$ than $B_2$ to output an $100\%$ accurate ensemble, then the latter should be preferred over the former as using a less powerful $\H$ can prevent overfitting. 
The following theorems answer these questions. 

\begin{theorem}
\label{thm:wl_condition}
The following is a necessary and sufficient condition on $\H$ that ensures that any maximizer of Equation~\eqref{eqn:robust_boosting_game} achieves $100\%$ adversarial accuracy on $S$: for any probability distribution $P'$ over points in the set \mbox{$S_{\text{aug}} \coloneqq \{(\x,y, y', \delta): (\x,y)\in S, y' \in \cY\setminus\{y\}, \delta \in \gB(\epsilon)\}$}, there exists a classifier $h\in \H$ which achieves slightly-better-than-random performance on $P'$
\begin{align*}
    &\mathbb{E}_{(\x,y, y', \delta) \sim P'}[\indct{h(\x+\delta)= y}]\\
    &\quad \geq  \mathbb{E}_{(\x,y, y', \delta) \sim P'}[\indct{h(\x+\delta)= y'}] + \tau.
\end{align*}
Here $\tau > 0$ is some constant.
\end{theorem}
\begin{theorem}
\label{thm:optimal_boosting_framework}
The margin-based boosting framework is \emph{optimal} in the following sense: there does not exist any other boosting framework which can guarantee a solution with $100\%$ adversarial accuracy with milder conditions on $\H$ than the above margin-based boosting framework.
\end{theorem}\vspace{-0.1in}
\textbf{Discussion.} 
The condition in Theorem~\ref{thm:wl_condition} holds if for any weighting of points in $S_{\text{aug}}$, there exists a base classifier in $\H$ which performs slightly better than \emph{random guessing}, where performance is measured with respect to an appropriate metric. 
In the case of binary classification, this metric boils down to ordinary classification accuracy, and the condition in Theorem~\ref{thm:wl_condition} can be rewritten as
\vspace{-2mm}
\begin{align*}
    \mathbb{E}_{(\x,y, y', \delta) \sim P'}[\indct{h(\x+\delta)= y}]\geq \frac{1+\tau}{2}.
\end{align*}
Such conditions are referred to as \emph{weak learning conditions} in the literature of boosting and have played a key role in the design and analysis of boosting algorithms~\citep{freund1995desicion, telgarsky2011fast, mukherjee2013theory}. 

Theorem~\ref{thm:optimal_boosting_framework} shows that the margin-based boosting framework is \emph{optimal} in the sense that among all boosting frameworks, the margin-boosting framework requires the weakest possible weak learning condition.
In a recent work, \citet{mukherjee2013theory} obtained a similar result in the context of standard multi-class classification.  
In particular, they develop optimal boosting frameworks and identify minimal weak learning conditions for standard multi-class classification.
Our work extends their results to the setting of adversarial robustness. 
Moreover, the results of \citet{mukherjee2013theory} can be obtained as a special case of our results by setting $\gB(\epsilon) = \{0\}.$
\vspace{0.05in}\\
\textbf{Comparison with \citet{abernethy2021multiclass}.}  
Recently, \citet{abernethy2021multiclass} developed a boosting framework for adversarial robustness. We now show that their framework is \emph{strictly} sub-optimal to our framework.
\citet{abernethy2021multiclass} require $\H$ to satisfy the following weak learning condition to guarantee that their boosting framework outputs an ensemble with $100\%$ adversarial accuracy: for any probability distribution $P'$ over points in the set \mbox{$\Tilde{S}_{\text{aug}} \coloneqq \{(\x,y, y'): (\x,y)\in S, y' \in \cY\setminus\{y\}\}$}, there exists a classifier $h\in \H$ which satisfies the following for some $\tau >0$:
\vspace{-2mm}
\begin{align*}
    &\mathbb{E}_{(\x,y, y') \sim P'}[\indct{\forall \delta \in \gB(\epsilon): h(\x+\delta)= y}]\\
    &\quad \geq  \mathbb{E}_{(\x,y, y') \sim P'}[\indct{\exists\delta \in \gB(\epsilon): h(\x+\delta)= y'}] + \tau.
\end{align*}




\begin{proposition}
\label{prop:wl_conditions_comparison}
Let $\text{WL}_{\robboostalg}$ denote the necessary and sufficient weak learning condition of our margin-boosting framework and  $\text{WL}_{\abernehty}$ denote the weak learning condition of the boosting framework of \citet{abernethy2021multiclass}. 
Then \mbox{$\text{WL}_{\abernehty} \implies \text{WL}_{\robboostalg}$.} The implication does not hold the other way round; that is, \mbox{$\text{WL}_{\robboostalg} \notimplies \text{WL}_{\abernehty}$.}

\end{proposition}


\vspace{-3mm}
\subsection{Robust Boosting Algorithm}
\label{sec:mrboost}
\vspace{-1mm}

\begin{algorithm}[tb]
   \caption{$\robboostalg$}
   \label{alg:mrrob}
\begin{algorithmic}[1]
  \small
  \STATE \textbf{Input:} training data $S$, boosting iterations $T$, learning rate $\eta$.
  \STATE Let $P_1$ be the uniform distribution over $S_{\text{aug}}$.
  \FOR{$t = 1 \dots T$}
  \STATE Compute $h_t\in \H$ as the minimizer of:
  \vspace{-0.08in}
  \begin{align*}
      \min_{h \in \H} \mathbb{E}_{(\x,y, y', \delta) \sim P_{t}}[\pairwiseloss{h(\x+\delta), y,y'}].
  \end{align*}
  \STATE  Compute probability distribution $P_{t+1}$, supported on $S_{\text{aug}}$, as:\vspace{-0.08in}
  \begin{align*}
      P_{t+1}(\x,y,y',\delta) \propto\exp\left(\eta\sum_{j=1}^{t}\pairwiseloss{h_j(\x+\delta),y,y'}\right),
  \end{align*}
   \ENDFOR
  \STATE \textbf{Output:} return the classifier $\amclftwo{h}{Q(T)}(\x)$, where  $Q(T)$ is the uniform distribution over $\{h_t\}_{t=1\dots T}$.
\end{algorithmic}
\end{algorithm}

In this section, we present our algorithm $\robboostalg$ for optimizing Equation~\eqref{eqn:robust_boosting_game}.
The pseudocode of this is shown in Algorithm~\ref{alg:mrrob}. 
Define the set $S_{\text{aug}}$ as $ \{(\x,y, y', \delta): (\x,y)\in S, y' \in \cY\setminus\{y\}, \delta \in \gB(\epsilon)\}$. To simplify the notation, we define the following pairwise 0-1 margin loss \vspace{-0.05in}
\begin{align*}
\pairwiseloss{h(\x), y,y'} \coloneqq \mathbb{I}(h(\x) \neq y) - \mathbb{I}(h(\x) \neq y').
\end{align*}\vspace{-0.15in}\\
In the $t$-th round of our algorithm, the following distribution $P_t$ is computed over $S_{\text{aug}}$:\vspace{-0.05in}
\vspace{-1mm}
  \begin{align*}
      P_{t}(\x,y,y',\delta) \propto\exp\Big(\eta\sum_{j=1}^{t-1}\pairwiseloss{h_j(\x+\delta), y,y'}\Big).
  \end{align*}\vspace{-0.15in}\\
Note that $P_t$ uses a distribution over adversarial perturbations rather than simply choose a single adversarial perturbation. Intuitively, for any given $(\x, y, y')$, this distribution places more weight on perturbations that are adversarial to the ensemble constructed till now, and less weight on perturbations that are non-adversarial to the ensemble. 
Once we have $P_t$, a new classifier $h_t$ is computed to minimize the weighted error relative to $P_t$ and added to the ensemble. Learning $h_t$ in this 
way helps us fix the mistakes of the past classifiers, and eventually leads to a robust ensemble.

To get a better understanding of $P_t$, we consider the following optimization problem. It can be easily shown that $P_{t}$ is an optimizer of this problem~\citep{catoni2004statistical, audibert2009fast}:
\vspace{-5mm}
\begin{align*}
    \max_{P' \in \Delta(S_{\text{aug}})} \Eover{ P'}{\sum_{j=1}^{t-1}\pairwiseloss{h_j(\x+\delta), y,y'}} -  \frac{KL(P'||P_1)}{\eta}
\end{align*}\vspace{-0.1in}\\
where $\Delta(S_{\text{aug}})$ is the set of all probability distributions over $S_{\text{aug}}$, and $KL(P'||P_1)$ is the KL divergence between $P',P_1$. Here, $P_1$ is the uniform distribution over $S_{\text{aug}}$. Without the KL term, $P_t$ would have placed all its weight on the worst possible perturbations (\emph{i.e.,} perturbations which fool the ensemble the most). However, the presence of KL term makes $P_t$ assign ``soft weights'' to points in $S_{\text{aug}}$ based on how poorly they are classified by the existing ensemble. This regularization actually plays a key role in the convergence of our algorithm to an optimal ensemble.

\begin{figure}[t]
    \centering
    \includegraphics[scale=0.18]{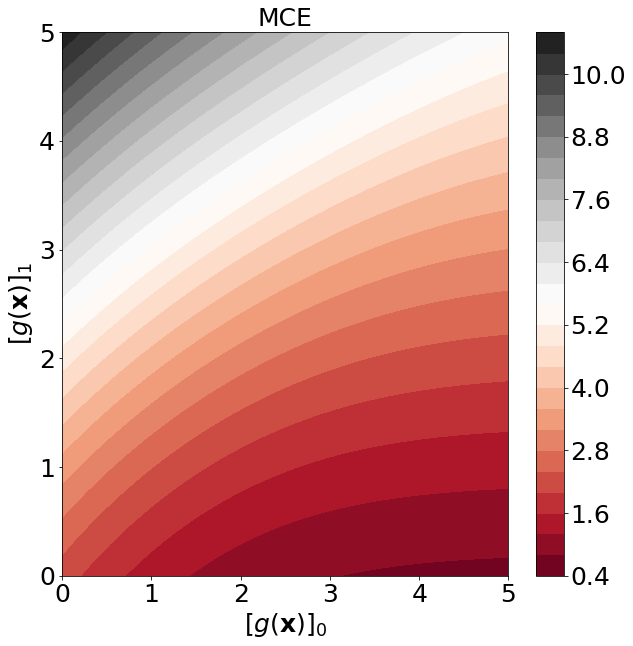}
    \includegraphics[scale=0.18]{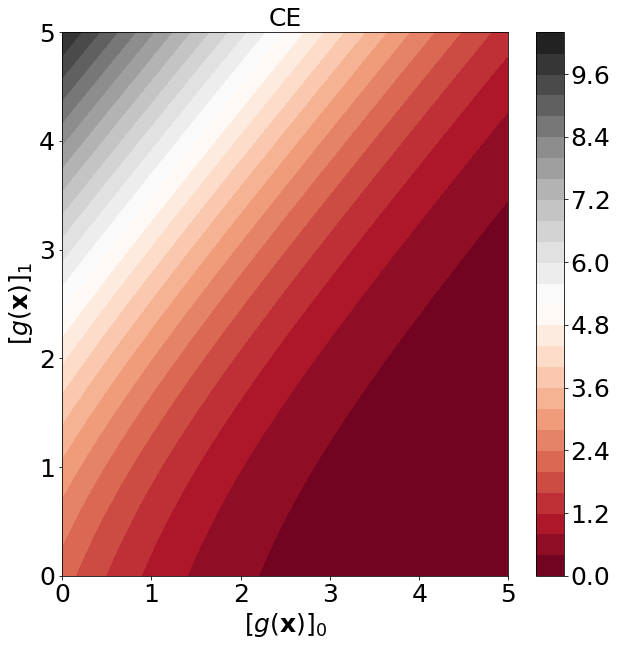}
    \caption{Contour plots of $\cemargin(g(\x),y,y')$ and $2\times\ce(g(\x),y)$ for $K=3, y = 0, y' =1$. The plots show how the loss functions vary with $[g(\x)]_0, [g(\x)]_1$ when we set $[g(\x)]_{2}$ to $0$. It can be seen that the two losses differ significantly on the right side of their plots where $\cemargin$ aggressively penalizes points whose $[g(\x)]_0, [g(\x)]_1$ are close to each other.  
    }
    \label{fig:MCEvsCE}
\end{figure}

Our algorithm has its roots in the framework of online learning~\citep{hazan2016introduction}.
We relegate related details to Appendix~\ref{sec:mrrob_design}.
One important thing to note here is that, when $\gB(\epsilon) = \{0\}$, our algorithm boils down to AdaBoost in binary classification setting, and to AdaBoost.MR in multi-class  setting~\citep{schapire1999improved}.\footnote{$\robboostalg$ and AdaBoost only differ in the choice of $\eta$.} Despite this connection, the analysis and implementation of our algorithm is significantly harder than AdaBoost. 
This is  because $\gB(\epsilon)$ is an infinite set in the adversarial setting, which thus forms a much more challenging zero-sum game.

We now present the following Theorem, which characterizes the rate of convergence of our Algorithm.
\begin{theorem}
\label{thm:convergence_rate_robboost}
Suppose Algorithm~\ref{alg:mrrob} is run with $\eta \leq \frac{1}{2\sqrt{T}}$. Then the ensemble $h_{Q(T)}$ output by the algorithm satisfies:
\vspace{-1mm}
\begin{align*}
    \max_{Q\in \Delta(\H)} \marginrob{Q, S} \leq \marginrob{Q(T), S} + \xi(T),
\end{align*}
where $\xi(T) = 3|\log{\left(nK\text{Vol}(\gB(\epsilon))\right)}|T^{-1/2}$. Moreover, the mixture distribution $P_{\text{avg}} = \sum_{t=1}^T \frac{1}{T}P_t$ satisfies
\vspace{-1mm}
\begin{align*}
    &\max_{P \in \Delta(S_{\text{aug}})}\Eover{(\x,y,y',\delta)\sim P}{\sum_{t=1}^T\frac{1}{T}{\pairwiseloss{h_t(\x+\delta), y, y'}}}\\
    &\leq \Eover{(\x,y,y',\delta)\sim P_{\text{avg}}}{\sum_{t=1}^T\frac{1}{T}{\pairwiseloss{h_t(\x+\delta), y, y'}}} + \xi(T).
\end{align*}
\vspace{-2mm}
\end{theorem}
\vspace{-2mm}
For $\ell_{p}$ balls, $\text{Vol}(B(\epsilon)) = O(\epsilon^{-d})$, and as a result $\xi(T) = O(d\log{(nK/\epsilon})T^{-1/2})$.   
This shows that the ensemble $h_{Q(T)}$ output by the algorithm is $O(dT^{-1/2})$ close to the max-margin solution for $\ell_p$ ball perturbations. The Theorem also shows that the mixture distribution $P_{\text{avg}}$ concentrates around the worst adversarial examples of $h_{Q(T)}$.


\subsection{Practical $\robboostalg$  for Neural Networks}

\begin{algorithm}[tb]
\caption{$\robboostalgnn$}
\label{alg:mrrob_nn}
\begin{algorithmic}[1]
  \small
  \STATE \textbf{Input:} training data $S$, boosting iterations $T$, learning rate $\eta$, SGD iterations $E$, SGD step size $\gamma$, sampling sub-routine: $\sampler$.
  \FOR{$t = 1 \dots T$}
  \STATE  $\theta_t \leftarrow  \begin{cases} \text{random initialization} & (\randinit) \\
  \theta_{t-1} \quad & (\perinit)\\
  \end{cases}$ 
  \FOR{$e = 1 \dots E$}
  \STATE Generate mini-batch
  \begin{align*}
      \{(\x_b, y_b, y_b', \delta_b)\}_{b=1}^B \leftarrow \sampler\left(S, \{\theta_j\}_{j=1}^{t-1}, \eta\right)
  \end{align*}
  \STATE Update $g_{\theta_t}$ using SGD:\vspace{-0.1in}
   \begin{align*}
  \theta_t \leftarrow \theta_t - \frac{\gamma}{B} \sum_{b=1}^B\nabla_{\theta} \cemargin(g_{\theta_t}(\x_b+\delta_b), y_b, y_b').
  \end{align*}\vspace{-0.1in}
  \ENDFOR
   \ENDFOR
  \STATE \textbf{Output:} Let $Q(T)$ be the uniform distribution over $\{g_{\theta_t}\}_{t=1\dots T}$. Output the classifier $\amclftwo{g}{Q(T)}(\x)$. 
\end{algorithmic}
\end{algorithm}

\label{sec:tract_rob_boosting}
Note that lines $4, 5$ of Algorithm~\ref{alg:mrrob} are computationally expensive to implement, especially when $\H, \gB(\epsilon)$ are not finite element sets. This intractability arises due to the presence of the discontinuous margin loss $\pairwiseloss{h(\x),y,y'}$  in the optimization and sampling objectives. We now attempt to make this algorithm computationally tractable by replacing $\pairwiseloss{h(\x),y,y'}$  with a smooth and differentiable surrogate loss. 
We focus on neural network score-based base classifiers: $\G = \{g_{\theta}:\theta \in \Theta \subseteq \R^D\},$ where $g_{\theta}:\cX \to \R^K$ is a neural network parameterized by $\theta$. 

To begin with, we introduce the following differentiable surrogate for pairwise margin loss $\pairwiseloss{h(\x),y,y'}$ (which serves a key role in our boosting framework), and which we refer to as margin cross entropy loss:\vspace{-0.03in}
\begin{align*}
\label{eqn:mce}
    \cemargin(g_{\theta}(\x),y,y') \coloneqq \ce(g_{\theta}(\x),y) + \ce(-g_{\theta}(\x),y').
\end{align*}\vspace{-0.15in}\\

For binary classification, $\cemargin$ is equivalent to $\ce$ (to be precise, $\cemargin = 2 \times \ce$). 
However, both the losses differ when $K > 2$. 
Our margin cross entropy loss encourages $[g_{\theta}(\x)]_y$ to be large while simultaneously forcing $[g_{\theta}(\x)]_{y'}$ to be small, thus increasing the pairwise margin between the two logits. This is in contrast to standard cross entropy loss which doesn't have an explicit term for $y'$ and doesn't necessarily increase the pairwise margin (see Figure~\ref{fig:MCEvsCE}). In our experiments, we show that $\cemargin$ loss is interesting even outside the context of margin boosting. 
In particular, we show that using $\cemargin$ in the objectives of existing defenses 
significantly improves their performance.



\begin{remark}
Several recent works have showed that performing logistic regression on linearly separable data leads to max-margin solutions~\citep{soudry2018implicit}. At a first glance, these results seem to suggest that $\ce$ loss suffices for max-margin solutions. However, it should be noted that these works study linear classifiers in the binary classification setting. It is not immediately clear if these results extend to non-linear classifiers in the multi-class setting. 
\end{remark}

\begin{algorithm}[t]
\caption{$\samplerexp$ (Exponential)}
\label{alg:sampler_exp}
\begin{algorithmic}[1]
  \small
  \STATE \textbf{Input:} training data $S$, base models $\{\theta_j\}_{j=1}^{t-1}$, learning rate $\eta$.
  \STATE Randomly sample a batch of points $\{(\x_b, y_b, y_b', \delta_b)\}_{b=1}^B$ from the following distribution: \vspace{-0.08in}
  \begin{align*}
      \hspace{-0.07in} P_{t}(\x,y,y',\delta) \propto \exp\left(\eta \cemargin\left(\sum_{j=1}^{t-1}g_{\theta_j}(\x+\delta), y, y'\right)\right).
  \end{align*}
  \STATE \textbf{Output:}  $\{(\x_b, y_b, y_b', \delta_b)\}_{b=1}^B$
\end{algorithmic}
\end{algorithm}
\begin{algorithm}[tb]
\caption{$\samplerall$}
\label{alg:sampler_all}
\begin{algorithmic}[1]
  \small
  \STATE \textbf{Input:} training data $S$, base models $\{\theta_j\}_{j=1}^t$.
  \STATE  $\widehat{S}_B \leftarrow \{\}$
  \STATE Uniformly sample a batch of points $\{(\x_b, y_b)\}_{b=1}^B$ from $S$.
  \FOR{$b = 1 \dots B$}
  \STATE compute $\delta_b$ as
  \begin{align*}
  &\delta_b \in \argmax_{\delta \in \gB(\epsilon)}\sum_{y'\in \cY\setminus\{y_b\}} \cemargin\left(\sum_{j=1}^{t}g_{\theta_j}(\x_b+\delta), y_b, y'\right)
  \end{align*}
  \STATE $\widehat{S}_B \leftarrow \widehat{S}_B \cup \{(\x_b, y_b, y', \delta_b)\}_{y' \in \cY \setminus \{y_b\}}.$
  \ENDFOR
  \STATE \textbf{Output:}  $\widehat{S}_B$
\end{algorithmic}
\vspace*{-0.5mm}
\end{algorithm}



\begin{table*}[t]
\centering
\caption{
\centering
Experiments with ResNet-18 on different datasets. $\dagger$ denotes that the results are from the 
last epoch checkpoint.
}
\label{tab:res18_results}
\centering
\begin{center}
\scriptsize
\begin{sc}
\begin{tabular}{c|cccccccccccc}
\toprule 
\multirow{2}{*}{Method} & \multicolumn{4}{c}{SVHN} &  \multicolumn{4}{c}{CIFAR-10} &  \multicolumn{4}{c}{CIFAR-100} \\
\cmidrule(lr){2-5} \cmidrule(lr){6-9} \cmidrule(lr){10-13} & Clean & Adv & Clean$^\dagger$ & Adv$^\dagger$ & Clean & Adv & Clean$^\dagger$ & Adv$^\dagger$ & Clean & Adv & Clean$^\dagger$ & Adv$^\dagger$ \\
\midrule 
\midrule 
AT & $89.07$ & $52.09$ & $90.60$ & $42.76$ & $\mathbf{82.06}$ & $50.95$  & $84.38$ & $46.39$  & $54.49$ & $27.09$ & $57.22$ & $22.96$ \\
AT + MCE & $\mathbf{89.69}$ & $\mathbf{53.25}$ & $\mathbf{90.89}$ & $\mathbf{46.78}$ & $\mathbf{}81.13$ & $\mathbf{51.86}$ & $84.37\mathbf{}$ & $\mathbf{48.47} $ & $54.37$ & $\mathbf{28.25}$ & $57.01$ & $\mathbf{24.93}$ \\
\bottomrule
\end{tabular}
\end{sc}
\end{center}
\vspace{-0.5cm}
\end{table*}
\begin{table*}[th]
\caption{
Experiments with WideResNet-34-10 on CIFAR10.}
\label{tab:wrn34_results}
\centering
\begin{small}
\begin{sc}
\setlength{\tabcolsep}{3mm}
\begin{tabular}{l|cccccc}
\toprule
Method & Clean & FGSM & CW & PGD-20 & PGD-100 & AutoAttack  \\
\midrule
\midrule
AT & $\mathbf{86.31}$ & $64.01$ & $53.28$ & $54.12$ & $53.75$ & $50.13$  \\
AT + MCE &  $85.56$ & $64.20$ & $53.46$ & $\mathbf{55.40}$ & $\mathbf{55.14}$ & $\mathbf{52.07}$ \\
\midrule
TRADES & $83.25$ & $62.48$ & $49.51$ & $54.97$ & $54.80$ & $51.92$ \\
TRADES + MCE & $84.76$ & $\mathbf{64.63}$ & $49.49$ & $\mathbf{56.23}$ & $\mathbf{55.99}$ & $52.40$ \\
\midrule
MART & $83.12$ & $63.68$ & $52.57$ & $55.75$ & $55.49$ & $50.85$ \\
MART + MCE & $\mathbf{83.65}$ & $\mathbf{64.3}$ & $\mathbf{54.24}$ & $\mathbf{56.31}$ & $\mathbf{56.15}$ & $\mathbf{52.81}$ \\
\midrule
GAIR & $83.91$ & $65.79$ & $49.44$ & $58.99$ & $58.97$ & $44.04$ \\
GAIR + MCE & $\mathbf{84.55}$ & $\mathbf{67.96}$ & $\mathbf{49.94}$ & $\mathbf{61.79}$ & $\mathbf{61.93}$ & $44.22$ \\
\midrule
AWP &  $\mathbf{85.32}$ & $65.89$ & $55.40$ & $57.37$ & $57.08$ & $53.67$\\
AWP + MCE & $84.97$ & $\mathbf{66.53}$ & $\mathbf{56.23}$ & $\mathbf{58.40}$ & $\mathbf{58.12}$ & $\mathbf{54.69}$ \\
\bottomrule
\end{tabular}
\end{sc}
\end{small}
\vspace*{-3mm}
\end{table*}

We now get back to Algorithm~\ref{alg:mrrob} and replace the margin loss $\pairwiseloss{h(\x),y,y'}$ in the Algorithm with the surrogate loss $\cemargin$. In particular, we replace line 4 in Algorithm~\ref{alg:mrrob} with the following
\begin{align*}
    \min_{\theta} \mathbb{E}_{(\x,y, y', \delta) \sim P_{t}}[\cemargin(g_{\theta}(\x+\delta), y,y')].
    \vspace{-2mm}
\end{align*}
We solve this optimization problem using SGD. 
In each iteration of SGD, we randomly sample a mini-batch according to the probability distribution $P_{t}$ and descend along the gradient of the mini-batch. 
One caveat here is that sampling from $P_{t}$ can be computationally expensive due to the presence of $\pairwiseloss{h(\x),y,y'}$ loss. 
To make this tractable, we 
replace it with $\cemargin$ and 
sample from the following distribution\vspace{-0.1in}\\
\begin{align*}
      &P_{t}(\x,y,y',\delta) \propto \exp\left(\eta \cemargin\left(\sum_{j=1}^{t-1}g_{\theta_j}(\x+\delta), y, y'\right)\right).
  \end{align*}\vspace{-0.15in}\\
Algorithm~\ref{alg:mrrob_nn}, when invoked with Algorithm~\ref{alg:sampler_exp} as the sampling sub-routine, describes this procedure. Notice that in line 3 of Algorithm~\ref{alg:mrrob_nn}, we consider two different initializations for the $t$-th base classifier $\theta_t$: $\randinit$, and $\perinit$. In $\randinit$, we randomly initialize $\theta_t$ using techniques such as Xavier initialization. In $\perinit$ (Persistent Initialization), we initialize $\theta_t$ to $\theta_{t-1}$. This makes $\theta_t$ stand on the shoulders of its precursors, and benefits its training.

While this algorithm is computationally tractable, it can be further improved by relying on the following heuristics:

\vspace{-1mm}
\begin{itemize}[leftmargin=*]
    \item \textbf{Aggressive Defense.} 
    Note that $P_t$ in $\samplerexp$ (Algorithm~\ref{alg:sampler_exp}) relies only on $\{\theta_j\}_{j=1}^{t-1}$ and not on $\theta_t$. 
    That is, the $t$-th base classifier ${\theta_t}$ is trained to be robust only to the ``soft'' adversarial examples generated from the classifiers $\{{\theta_j}\}_{j=1}^{t-1}$.
    In our experiments, we noticed that making $\theta_t$ to be also robust to its own adversarial examples makes the optimization more stable and improves its convergence speed. So, we make our sampler rely on $\theta_t$ as well.
    \vspace{-0.05in}
    \item \textbf{Efficient Samplers.} The key computational bottleneck in our algorithm is the sampler in Algorithm~\ref{alg:sampler_exp}. In order to scale up our algorithm to large datasets, we replace the sampling sub-routine in Algorithm~\ref{alg:sampler_exp} with a more efficient sampler described in Algorithm~\ref{alg:sampler_all}. This sampler approximates the ``soft weight'' assignments via appropriate ``hard weights''.
    To be precise, we first uniformly sample $(\x,y)$ from $S$ and find a perturbation whose pairwise margin loss w.r.t all classes $y' \in \cY\setminus \{y\}$ is the highest, and use it to train our base classifier.  In our experiments, we tried two other samplers - $\samplerrnd, \samplermax$ - which differ in the way they perform the hard weight assignment (see Appendix~\ref{sec:samplers_efficient} for a more thorough discussion on these samplers). We noticed that $\samplerall$ is the best performing variant among the three, and we use it in all our experiments.   \vspace{-0.05in}
\end{itemize}

We provide a PyTorch-style pseudocode of our algorithm ($\samplerall$ variant) for training a single network: 

\definecolor{codegreen}{rgb}{0,0.6,0}
\definecolor{codegray}{rgb}{0.5,0.5,0.5}
\definecolor{codepurple}{rgb}{0.58,0,0.82}
\definecolor{backcolour}{rgb}{0.95,0.95,0.92}
\lstdefinestyle{mystyle}{
    backgroundcolor=\color{backcolour},   
    commentstyle=\color{codegreen},
    keywordstyle=\color{magenta},
    numberstyle=\tiny\color{codegray},
    stringstyle=\color{codepurple},
    basicstyle=\ttfamily\scriptsize,
    breakatwhitespace=false,         
    breaklines=true,                 
    captionpos=b,                    
    keepspaces=true,                 
    numbersep=5pt,                  
    showspaces=false,                
    showstringspaces=false,
    showtabs=false,                  
    tabsize=2
}
\lstset{style=mystyle}
\begin{lstlisting}[language=Python]
# net: classification neural network 
# attack: Sampler.ALL adversarial attack (uses PGD)
# num_classes: the number of classes for the task

import torch.nn.functional as F
for inputs, targets in trainloader:
    adv_inp = attack(net, inputs, targets, num_classes)
    
    # standard adversarial training
    adv_outputs = net(adv_inp)
    loss = F.cross_entropy(adv_outputs, targets)
    
    if use_MCE:  # our method
        loss2 = - F.log_softmax(-adv_outputs, dim=1)
        loss2 *= (1 - F.one_hot(targets, num_classes))
        loss2 = loss2.mean() / (num_classes - 1)
        loss = loss + loss2
    
    # optimization step
    optimizer.zero_grad()
    loss.backward()
    optimizer.step()
\end{lstlisting}
As can be seen from the pseudocode, our algorithm does not need extra forward passes.
It only performs a few simple operations in the logit (``adv\_outputs") space. The $\samplerall$ attack ($1^{\text{st}}$ line in \texttt{for} loop) has similar runtime as the standard PGD attack and can be implemented using similar logic as in the pseudocode above. 
Therefore, our algorithm has only a slight computational overhead over baselines 
(see a comparison in Section~\ref{sec:efficiency}).






\vspace*{-2mm}
\section{Experiments}
\label{sec:exps}

\begin{table*}[th]
\caption{Boosting experiments with ResNet-18 being the base classifier.
}
\label{tab:boosting}
\centering
\begin{small}
\begin{sc}
\begin{tabular}{l|cccccccccc}
\toprule
\multirow{2}{*}{ Method } & \multicolumn{2}{c}{ Iteration 1 } & \multicolumn{2}{c}{ Iteration 2 }  & \multicolumn{2}{c}{ Iteration 3 }  & \multicolumn{2}{c}{ Iteration 4 }  & \multicolumn{2}{c}{ Iteration 5 } \\
&  Clean & Adv&  Clean & Adv&  Clean & Adv&  Clean & Adv&  Clean & Adv \\
\midrule
\midrule
Wider model & $82.61$ & $ 51.73$ & --- & --- & --- & --- & --- & --- & --- & ---  \\
Deeper model &  $82.67$ & $52.32$  & --- & --- & --- & --- & --- & --- & --- & ---  \\
\midrule
$\abernehty + \randinit$ & $82.00$ & $51.05$ & $84.58$ & $49.95$ & $83.87$ & $51.66$ & $82.56$ & $52.72$ & $81.44$ & $52.92$ \\
$\abernehty+\perinit$ & $82.18$ & $50.97$ & $85.60$ & $50.13$ & $84.59$ & $51.77$ & $84.21$ & $52.79$ & $82.78$ & $53.28$ \\
$\robboostalgnn + \randinit$ & $81.04$ & $51.83$ & $84.61$ & $52.68$ & $84.93$ & $53.51$ & $85.01$ & $53.95$ & $85.35$ & $54.13$ \\
$\robboostalgnn+\perinit$ & $81.34$ & $51.92$ & $84.97$ & $52.97$ & $85.28$ & $53.62$ & $85.99$ & $54.26$ & $86.16$ & $54.42$ \\
\bottomrule
\end{tabular}
\end{sc}
\end{small}
\end{table*}

In this section, we present experimental results showing the effectiveness of the proposed boosting technique. In addition, we demonstrate the efficacy of the proposed loss ($\cemargin$) for standard adversarial training.
Our focus here is on the $\ell_\infty$ threat model (\emph{i.e.,} $\gB(\epsilon)$ is the $\ell_{\infty}$ norm ball).

\subsection{Effectiveness of $\cemargin$}
\label{sec:single_model_experiment}

We first demonstrate the effectiveness of $\cemargin$ for training a single robust model. Here, we compare AT~\citep{madry2017towards} with AT + MCE, which is a defense technique obtained by replacing $\ce$ in AT with $\cemargin$, and relying on $\samplerall$ to generate mini-batches during training. We consider three datasets: SVHN, CIFAR-10 and CIFAR-100.
We train ResNet-18 using SGD with $0.9$ momentum for $100$ epochs.
The initial leaning rate is set to $0.1$ and it is further decayed by the factor of $10$ at the $50$-th and $75$-th epoch.
The batch size is set to $128$ in this work.
We also use a weight decay of $5 \times 10^{-4}$.
For the $\ell_\infty$ threat model, we use $\epsilon=8/255$.
The step size of attacks is $1/255$ for SVHN and $2/255$ for CIFAR-10 and CIFAR-100.
PGD-10~\citep{madry2017towards} attack is used for adversarial training and PGD-20 is used during testing period.
In Table~\ref{tab:res18_results}, we report both the best checkpoint model as well as the last checkpoint model, the gap between which is known as the robust overfitting phenomenon \citep{Rice2020OverfittingIA}.
For each checkpoint, we report its clean accuracy and adversarial accuracy.
We use bold numbers when the accuracy gap is $>0.20\%$.
We can see that AT + MCE consistently improves the robustness of AT across different datasets, with only a tiny drop in clean accuracy.
This indicates that $\cemargin$ is broadly applicable even outside the context of boosting.  Note that although the computation of adversarial perturbations in $\samplerall$ involves all the classes $y'\in \cY \setminus \{y\}$, it can be computed with only one forward pass, resulting in less than $5\%$ increase in runtime.

We now train larger capacity models,
as they lead to state-of-the-art results in the literature.
Concretely, we use WideResNet-34-10 on CIFAR-10 and use the same setting as \citet{madry2017towards}.
We consider a number of state-of-the-art baseline algorithms:
1) AT \citep{madry2017towards};
2) TRADES \citep{zhang2019theoretically};
3) MART \citep{wang2019improving};
4) AWP \citep{Wu2020AdversarialWP},
and
5) GAIR \citep{Zhang2020GEOMETRYAWAREIA}.
We combine these defenses with $\cemargin$ to improve their robustness (see Appendix~\ref{sec:mce_variants} for more details).
We use the same hyperparameters as stated previously, to train these defenses.
With regard to attacking the trained models, we choose FGSM, CW$_\infty$ \citep{Carlini2017TowardsET}, PGD-20 and PGD-100 attacks with $\epsilon=8/255$.
Additionally, we test the model performance against the AutoAttack \citep{Croce2020ReliableEO}, which is a strong and reliable evaluation suite;
stable performances under such threat models often show that the robustness of our defenses is not caused by the ``obfuscated gradient'' phenomenon \citep{Athalye2018ObfuscatedGG}.
The results from this experiment are reported in Table \ref{tab:wrn34_results}.
It can be seen that the robustness of all the baseline algorithms can be improved using $\cemargin$, demonstrating the effectiveness of the proposed loss.
We also notice that for GAIR, there is a remarkable drop in performance on the AutoAttack. 
This suggests that some form of obfuscated gradients could potentially be the reason behind its robustness.
            
\vspace*{-2mm}
\subsection{Effectiveness of Margin-Boosting}
We now present experimental results showing the effectiveness of $\robboostalgnn$. In our experiments, we use $\samplerall$ to generate mini-batches (see  Appendix~\ref{sec:more_results} for more details). 
We consider the following baselines:
1) larger models trained end-to-end using AT, and 
2) $\abernehty$, which is the boosting algorithm of~\citet{abernethy2021multiclass} (see Appendix~\ref{sec:abernethy} for the details).
For the boosting techniques, we set the total number of boosting iterations to $5$, and choose ResNet-18 as our base classifier. For end-to-end trained larger models, we consider:  1) an ensemble of five ResNet-18 models (``wider model''), and 2) a ResNet-152 model which has slightly larger number of parameters than the wider model (``deeper model''). 
 The hyperparameter settings, including the threat model setup and optimization schedule for each boosting iteration, are identical to those used in the earlier experiments. 

Table~\ref{tab:boosting} presents the results from this experiment. 
For all the techniques, we report the best robustness checkpoint results after each boosting iteration (for end-to-end trained models, number of boosting iterations is $1$, and longer training could not bring further improvement \citep{Pang2021BagOT}).
Several conclusions can be drawn from Table \ref{tab:boosting}.
Firstly, $\robboostalgnn$ has significantly better performance than all the baselines, on both clean and adversarial accuracies. This shows the effectiveness of our margin-boosting framework over the boosting framework of \citet{abernethy2021multiclass}. It also shows that boosting techniques can outperform end-to-end trained larger networks. 
Next, persistent initialization ($\perinit$) not only improves $\robboostalgnn$, but also $\abernehty$ (\citet{abernethy2021multiclass} only study $\randinit$ in their work). This makes it a helpful technique for boosting in the context of deep learning.

\vspace*{-2mm}
\section{Conclusion and Future Work
}
\label{sec:conclusion}
We proposed a margin-boosting framework for building robust ensembles. Theoretically, we showed that our boosting framework is optimal and derived a computationally efficient boosting algorithm ($\robboostalgnn$) from our framework. Through extensive empirical evaluation, we showed that our algorithm outperforms both existing boosting techniques and larger models trained end-to-end. 

\textbf{Future Work.}
We believe the performance of our algorithm can be further improved by designing better samplers, and in particular, faster techniques to implement Algorithm~\ref{alg:sampler_exp}. Next, it'd be interesting to understand, both theoretically and empirically, why boosting techniques outperform end-to-end trained larger models.

\section*{Acknowledgement}
We thank Chen Dan, Feng Zhu and Mounica for helpful discussions.
Dinghuai Zhang thanks the never-ending snow storm in Montreal for preventing him from any form of outdoor activity.
Hongyang Zhang is supported in part by an NSERC Discovery Grant.
Aaron Courville thanks the support of Microsoft Research, Hitachi and CIFAR.
Yoshua Bengio acknowledges the funding from CIFAR, Samsung, IBM and Microsoft.
Pradeep Ravikumar acknowledges the support of ARL and NSF via IIS-1909816.

\bibliography{ref}
\bibliographystyle{icml2022}

\newpage
\appendix
\onecolumn

\section{Notation}
\label{sec:app_notation}

\begin{center}
 \begin{tabular}{l | l} 
 \toprule
 Symbol & Description \\ [0.5ex] 
 \midrule
 $\x$ & feature vector \\ 
 $y$ & label \\
 $\cX$ & domain of feature vector \\ 
 $\cY$ & domain of the label \\
 $K$ & number of classes in multi-class classification problem\\
 $S$ & training data set\\
 $P$ & true data distribution\\
 $P_n$ & empirical distribution\\
 \midrule
  $h:\cX\to \cY$ &  standard classifier\\
 $g:\cX\to \R^K$ & score based classifier\\
 $\ell_{0-1}$ & $0/1$ classification loss\\
 $\ell$ & convex surrogate of $\ell_{0-1}$\\
 $\ce$ & cross-entropy loss\\
  $\cemargin $& margin cross-entropy loss\\
  $R(g; \ell)$ & population risk of classifier $g$, measured w.r.t $\ell$\\
 $\emprisk(g; \ell)$ & empirical risk of classifier $g$,  measured w.r.t $\ell$\\
 \midrule
 $\gB(\epsilon)$ & Set of valid perturbations of an adversary\\
  $S_{\text{aug}}$ & $\{(\x,y, y', \delta): (\x,y)\in S, y' \in \cY\setminus\{y\}, \delta \in \gB(\epsilon)\}$\\
 $\Tilde{S}_{\text{aug}}$ & $\{(\x,y, y'): (\x,y)\in S, y' \in \cY\setminus\{y\}\}$\\
 $\advrisk(g; \ell)$ & population adversarial risk of classifier $g$, measured w.r.t $\ell$\\
 $\empadvrisk(g; \ell)$ & empirical adversarial risk of classifier $g$, measured w.r.t $\ell$\\
 \midrule
 $\H$ & compact set of base classifiers\\
 $\G$ & compact set of score-based base classifiers\\
 $\Delta(\H)$ & set of all probability distributions over $\H$\\
 $Q$ & probability distribution over base classifiers $\H$\\
 $h_{Q}$ & \begin{tabular}{@{}l@{}}ensemble constructed by placing probability distribution $Q$ \\over elements in $\H$\end{tabular}\\
 \bottomrule
\end{tabular}
\end{center}

\begin{table*}[th]
 \caption{Margin related terminology}
 \label{tab:margin_terminology}
\begin{center}
 \begin{tabular}{l | l} 
 \toprule
 Symbol & Description \\ [0.5ex] 
 \midrule
  $\pairwiseloss{h(\x), y,y'}$ & $ \mathbb{I}(h(\x) \neq y) - \mathbb{I}(h(\x) \neq y')$\\
  $\margin{Q, \x,y}$ & $[h_Q(\x)]_y - \max_{y' \neq y} [h_Q(\x)]_{y'}$\\
 $\margin{Q, S} $ & $\min_{(\x,y)\in S}\margin{Q, \x, y}$\\
 $\marginrob{Q, S}  $ & $\min_{(\x,y)\in S}\min_{\delta \in \gB(\epsilon)}\margin{Q, \x+\delta, y}$\\
 \bottomrule
 \end{tabular}
 \end{center}
 \end{table*}

\section{Proofs of Section~\ref{sec:wl_condition}}
\label{sec:wl_condition_proofs}
\subsection{Intermediate Results}
We first present the following intermediate result that helps us prove  Theorems~\ref{thm:wl_condition},~\ref{thm:optimal_boosting_framework} and Proposition~\ref{prop:wl_conditions_comparison}.
\begin{lemma}
\label{lem:equiv_wl_condition}
The following three statements are equivalent:
\begin{enumerate}
    \item There exists a $Q\in \Delta(\H)$ such that the ensemble $\amclftwo{h}{Q}$ achieves $100\%$ adversarial accuracy on the training set
    $S$.
    \item For any maximizer $Q^*$ of Equation~\eqref{eqn:robust_boosting_game}, the ensemble $\amclftwo{h}{Q^*}$ achieves $100\%$ adversarial accuracy on
    the training set 
    $S$.
    \item The hypothesis class $\H$ satisfies the following condition: for any probability distribution $P'$ over points in the set \mbox{$S_{\text{aug}} \coloneqq \{(\x,y, y', \delta): (\x,y)\in S, y' \in \cY\setminus\{y\}, \delta \in \gB(\epsilon)\}$}, there exists a classifier $h\in \H$ which achieves slightly-better-than-random performance on $P'$
\begin{align*}
    \mathbb{E}_{(\x,y, y', \delta) \sim P'}[\indct{h(\x+\delta)= y}] \geq  \mathbb{E}_{(\x,y, y', \delta) \sim P'}[\indct{h(\x+\delta)= y'}] + \tau.
\end{align*}
Here $\tau > 0$ is some constant.
\end{enumerate}
\end{lemma}
\begin{proof}
To prove the Lemma, it suffices to show that $(1) \iff (2), (2) \iff (3)$.
\begin{itemize}
    \item \textbf{Proof of $(1) \implies (2)$.} Let $Q$ be the ensemble which achieves $100\%$ adversarial accuracy on $S$. We first show that this is equivalent to: $\marginrob{Q,S} > 0$. To see this, first note that the following should hold for any $(\x,y) \in S$
    \[
    \amclftwo{h}{Q}(\x+\delta) = y, \quad \text{for all }\delta \in \gB(\epsilon).
    \]
     Now consider the following, for any $(\x,y) \in S$
    \begin{align*}
        &\forall \delta \in \gB(\epsilon), \ \amclftwo{h}{Q}(\x+\delta) = y \\
        & \stackrel{(a)}{\iff} \forall \delta \in \gB(\epsilon), \ [h_Q(\x+\delta)]_y > \max_{y'\neq y}  [h_Q(\x+\delta)]_{y'}\\
        & \iff \min_{\delta \in \gB(\epsilon)}  [h_Q(\x+\delta)]_y - \max_{y' \neq y} [h_Q(\x+\delta)]_{y'} > 0\\
        & \iff \min_{\delta \in \gB(\epsilon)} \margin{Q, \x+\delta,y} > 0,
    \end{align*}
    where $(a)$ follows from the definition of $\amclftwo{h}{Q}(\x).$   
    Since the last statement in the above display holds for any $(\x,y) \in S$, we have
    \begin{align*}
    &\min_{(\x,y)\in S}\min_{\delta \in \gB(\epsilon)} \margin{Q, \x+\delta,y} > 0\\
    &\iff \marginrob{Q, S} > 0.
    \end{align*}
    This shows that statement $(1)$ is equivalent to: $\marginrob{Q, S} > 0$. 
    Now, let $Q^*$ be any maximizer of Equation~\eqref{eqn:robust_boosting_game}. Since $\marginrob{Q^*, S} \geq \marginrob{Q, S} $ and $\marginrob{Q, S} > 0$, we have
    \[
    \marginrob{Q^*, S} > 0.
    \]
    From our above argument, we know that this is equivalent to saying that the ensemble $\amclftwo{h}{Q^*}$ achieves $100\%$ adversarial accuracy on training set $S$. This shows that $(1) \implies (2).$
    \item \textbf{Proof of $(2) \implies (1)$.} This statement trivially holds.  
    \item \textbf{Proof of $(2) \iff (3)$.} Suppose $\H$ satisfies the weak learning condition stated in statement (3). Then we have
    \begin{align*}
        &\forall P' \in \Delta(S_{\text{aug}}), \exists h \in \H, \text{ such that } \mathbb{E}_{(\x,y, y', \delta) \sim P'}[\indct{h(\x+\delta)= y} - \indct{h(\x+\delta)= y'}] \geq \tau\\
        &\iff \forall P' \in \Delta(S_{\text{aug}}),\   \max_{h \in \H}\mathbb{E}_{(\x,y, y', \delta) \sim P'}[\indct{h(\x+\delta)= y} - \indct{h(\x+\delta)= y'}] \geq \tau\\
        & \iff \min_{P' \in \Delta(S_{\text{aug}})} \max_{h \in \H}\mathbb{E}_{(\x,y, y', \delta) \sim P'}[\indct{h(\x+\delta)= y} - \indct{h(\x+\delta)= y'}] \geq \tau\\
        & \stackrel{(a)}{\iff} \min_{P' \in \Delta(S_{\text{aug}})} \max_{Q \in \Delta(\H)}\Eover{(\x,y, y', \delta) \sim P', h \sim Q}{\indct{h(\x+\delta)= y} - \indct{h(\x+\delta)= y'}} \geq \tau\\
        & \stackrel{(b)}{\iff} \max_{Q \in \Delta(\H)} \min_{P' \in \Delta(S_{\text{aug}})} \Eover{h \sim Q, (\x,y, y', \delta) \sim P'}{\indct{h(\x+\delta)= y} - \indct{h(\x+\delta)= y'}} \geq \tau\\
        & \iff \max_{Q \in \Delta(\H)} \min_{(\x,y, y', \delta) \in S_{\text{aug}}} \Eover{h \sim Q}{\indct{h(\x+\delta)= y} - \indct{h(\x+\delta)= y'}} \geq \tau,
    \end{align*}
    where $(a)$ follows since the optimum of a linear program over a convex hull can always be attained at an extreme point, and (b) follows from our compactness assumptions on $\H, \gB(\epsilon)$ and from Sion's minimax theorem which says that the max-min and min-max values of convex-concave games over compact domains are equal to each other~\cite{sion1958general}. (b) can also be obtained using Ky Fan's minimax theorem~\citep{fan1953minimax}. Continuing, we get
    \begin{align*}
        &\forall P' \in \Delta(S_{\text{aug}}), \exists h \in \H, \text{ such that } \mathbb{E}_{(\x,y, y', \delta) \sim P'}[\indct{h(\x+\delta)= y} - \indct{h(\x+\delta)= y'}] \geq \tau\\
        & \iff \max_{Q \in \Delta(\H)} \min_{(\x,y, y', \delta) \in S_{\text{aug}}} \Eover{h \sim Q}{\indct{h(\x+\delta)= y} - \indct{h(\x+\delta)= y'}} \geq \tau\\
        &\iff \max_{Q \in \Delta(\H)} \marginrob{Q,S} \geq \tau\\
        & \stackrel{(c)}{\iff} \max_{Q \in \Delta(\H)} \marginrob{Q,S} > 0,
    \end{align*}
    where the reverse implication in $(c)$ follows from our assumption that $\H, \gB(\epsilon)$ are compact sets~\citep{sion1958general}. In our proof for $(1) \implies (2)$ we showed that the last statement in the above display is equivalent to saying $\amclftwo{h}{Q}$ achieves  $100\%$ adversarial accuracy on $S$. This shows that $(2) \iff (3)$.
\end{itemize}

\end{proof}
\subsection{Proof of Theorem~\ref{thm:wl_condition}}
The proof of this Theorem follows directly from Lemma~\ref{lem:equiv_wl_condition}. In particular, the equivalence between statements $(2)$ and $(3)$ in the Lemma proves the Theorem.

\subsection{Proof of Theorem~\ref{thm:optimal_boosting_framework}}
The proof of this Theorem again follows from Lemma~\ref{lem:equiv_wl_condition}. 
Suppose there exists a boosting framework  which can guarantee $100\%$ accurate solution with a milder  condition on $\H$ than the condition in statement (3) of Lemma~\ref{lem:equiv_wl_condition}. Referring to this milder condition as statement 4, we have: statement $4$ $\implies$ statement $1$ in Lemma~\ref{lem:equiv_wl_condition}. Given the equivalence between statements 1 and 3 in Lemma~\ref{lem:equiv_wl_condition}, we can infer the following: statement 4 $\implies$ statement 3. This shows that any hypothesis class satisfying statement 4 should also satisfy  statement 3. Thus statement 4 can't be milder than the condition in statement 3. This shows that margin-boosting is optimal.

\subsection{Proof of Proposition~\ref{prop:wl_conditions_comparison}}
The first part of the proposition on proving \mbox{$\text{WL}_{\abernehty} \implies \text{WL}_{\robboostalg}$} follows from Theorem~\ref{thm:optimal_boosting_framework}. So, here we focus on proving \mbox{$\text{WL}_{\robboostalg} \notimplies \text{WL}_{\abernehty}$.} To prove this statement, it suffices to construct a base hypothesis class $\H$ which satisfies $\text{WL}_{\robboostalg}$ for some $\tau > 0$, but doesn't satisfy $\text{WL}_{\abernehty}$ for any $\tau > 0$. Here is how we construct such a $\H$. We consider the following simple setting:  $$\cX = \R, \cY = \{0,1\}, \gB(\epsilon) = \{\delta: |\delta|\leq \epsilon\}.$$ We let $\epsilon = 1$, and $S = \{(0, 1)\}$; that is, we only have $1$ sample in our training data with feature $x = 0$ and label $y = 1$. Our base hypothesis class is the set
$\H = \{h_{\theta}\}_{\theta \in [-1, 0.9]},$
where $h_{\theta}:\cX \to \cY$ is defined as 
\[
h_{\theta}(x) = \begin{cases} 0,\quad &\text{if } x\in [\theta, \theta+0.1] \\
1,\quad & \text{otherwise}
\end{cases}.
\]
In this setting, the weak learning condition $\text{WL}_{\abernehty}$ can be rewritten as follows: there exists a classifier $h_{\theta}\in \H$ which satisfies the following for some $\tau >0$
\begin{align*}
    &\indct{\forall \delta \in \gB(\epsilon): h_{\theta}(\delta)= 1}\geq \frac{1+\tau}{2}.
\end{align*}
Based on our construction of $\H$, it is easy to see that there is no $h\in \H$ which can satisfy the above condition for any $\tau > 0$. So, $\text{WL}_{\abernehty}$ is not satisfied in this setting. Now consider the weak learning condition $\text{WL}_{\robboostalg}$. This can be rewritten as follows: for any probability distribution $P'$ over points in the set $\gB(1)$, there exists a classifier $h_{\theta}\in \H$ which satisfies the following for some $\tau > 0$:
\begin{align*}
    &\mathbb{E}_{\delta \sim P'}[\indct{h_{\theta}(\delta)= 1}] \geq  \frac{1+\tau}{2}.
\end{align*}
We now show that for our choice of $\H$, the above weak learning condition holds for $\tau = 0.2$. To see this, divide $\gB(1)$ into the following (overlapping) intervals of width $0.1$: $[-1,-0.9], [-0.9,-0.8], \dots [0.8,0.9], [0.9,1]$. There are $20$ such intervals. It is easy to see that for any probability distribution $P'$ over $\gB(1)$, there exists atleast one interval in which the the probability distribution $P'$ has mass $<= \frac{1}{10}.$ By choosing a $h_{\theta}\in \H$ which assigns $0$ to points in that particular interval (and $1$ to all other points), we can show that the $\mathbb{E}_{\delta \sim P'}[\indct{h_{\theta}(\delta)= 1}] \geq  \frac{1.2}{2}.$ This shows that $\text{WL}_{\robboostalg}$ holds in this setting with $\tau = 0.2$. This shows that \mbox{$\text{WL}_{\robboostalg} \notimplies \text{WL}_{\abernehty}$.}

\section{Design of Algorithm~\ref{alg:mrrob}}
\label{sec:mrrob_design}
As previously mentioned, Algorithm~\ref{alg:mrrob_nn} has its roots in the framework of online learning~\citep{hazan2016introduction}. In this section, we present necessary background on online learning, game theory, and derive our algorithm for solving the max-min game in Equation~\eqref{eqn:robust_boosting_game}.
\subsection{Online Learning}
The online learning framework can be seen as a repeated game between a learner/decision-maker and an adversary. 
In this framework, in each round $t$, the learner makes a prediction \mbox{$\z_t \in \cZ$}, where $\cZ \subseteq\R^d$, and the adversary  chooses a loss function \mbox{$f_t:\cZ \rightarrow \mathbb{R}$} and observe each others actions. The goal of the learner is to choose a sequence of actions $\{\z_t\}_{t=1}^T$ so that the cumulative loss $\sum_{t=1}^T f_t(\z_t)$ is minimized. The benchmark with which the cumulative loss will be compared is called the best fixed policy in hindsight, which is given by $\inf_{\z \in \cZ}\sum_{t=1}^Tf_t(\z)$. This results in the following notion of \emph{regret}, which the learner aims to minimize $$\sum_{t = 1}^T f_t(\z_t) - \inf_{\z \in \cZ}\sum_{t=1}^Tf_t(\z).$$ 
Observe that there is a very simple strategy for online learning that guarantees $0$ regret, but under the provision that $f_t$ \emph{was known} to the learner ahead of round $t$. Then, an optimal strategy for the learner is to predict $\z_t$ as simply a minimizer of $f_t(\z).$ It is easy to see that this algorithm, known as Best Response (BR), has $0$ regret. While this is an impractical algorithm in the framework of online learning, it can be used to solve min-max games, as we will see in Section~\ref{sec:deriving_mrboost}.

A number of practical and efficient algorithms for regret minimization have been developed in the literature of online learning~\citep{hazan2016introduction, mcmahan2017survey, krichene2015hedge, kalai2005efficient}. In this work, we are primarily interested in the exponential weights update algorithm (Algorithm~\ref{alg:exp}).  In each round of this algorithm, the learner chooses its action $\z_t$ by uniformly sampling a point from the following distribution 
\[
P_t(\z) \propto \exp\left( -\eta\sum_{i = 1}^{t-1}f_i(\z)\right).
\]
\begin{algorithm}[tb]
\caption{Exponential Weights Algorithm (\textsc{EXP})}
\label{alg:exp}
\begin{algorithmic}[1]
  \small
  \STATE \textbf{Input:}   learning rate $\eta$
  \FOR{$t = 1 \dots T$}
  \STATE Sample $\z_t$ from the following distribution
  \[
  P_t(\z) \propto \exp\left( -\eta\sum_{i = 1}^{t-1}f_i(\z)\right).
  \]
  \STATE Play $\z_t$ and observe loss function $f_t$. Suffer loss $f_t(\z_t).$
  \ENDFOR
\end{algorithmic}
\end{algorithm}
We now present the following theorem which bounds the regret of this algorithm.
\begin{theorem}[Regret Bound]
\label{thm:exp_alg_regret_bound}
Suppose the action space $\cZ$ of the learner is a compact set. Moreover suppose the sequence of loss functions chosen by the adversary are measurable w.r.t Lebesgue measure, and uniformly bounded: $\sup_{t\in T, \z\in \cZ}|f_t(\z)| \leq B$. Suppose Algorithm~\ref{alg:exp} is run with $\eta \leq \frac{1}{2B\sqrt{T}}$. Then its expected regret can be bounded as follows
\[
\sup_{\z \in \cZ} \E\left[\sum_{t = 1}^T f_t(\z_t) - \sum_{t=1}^Tf_t(\z)\right]  \leq 2B\sqrt{T}(|\log\text{Vol}(\cZ)| + 1).
\]
\end{theorem}
\subsubsection{Intermediate Results for proof of Theorem~\ref{thm:exp_alg_regret_bound}}
Before we present the proof of Theorem~\ref{thm:exp_alg_regret_bound}, we present some intermediate results that help us prove the theorem. The following Lemma bounds the total variation distance between two probability distributions.
\begin{lemma}
\label{lem:exp_stability}
Let $\cZ\subseteq \R^d$ be a compact set. Let $f:\cZ\to \R$, $g:\cZ\to \R$ be  two  functions measurable w.r.t Lebesgue measure and are uniformly bounded: $\sup_{\z\in\cZ}|f(\z)|\leq B_f$, $\sup_{\z\in\cZ}|g(\z)|\leq B_g$. Let $P_1, P_2$ be two probability distributions over $\cZ$ that are defined as follows
\[
P_1(\z) \propto \exp\left(f(\x)\right), \quad P_2(\z) \propto \exp\left(f(\x) + g(\x)\right).
\]
Then the total variation distance between $P_1, P_2$ can be bounded as follows
$$\text{TV}(P_1, P_2) \leq \frac{\exp(2B_g)-1}{2}.$$
\end{lemma}
\begin{proof}
From the definition of total variation distance we have
\[
\text{TV}(P_1, P_2) = \frac{1}{2}\int_{\cZ}\left|\frac{\exp(f(\z))}{N_1} - \frac{\exp(f(\z) + g(\z))}{N_2}\right| d\z,
\]
where $N_1, N_2$ are normalization constants which are defined as 
$$N_1 = \int_{\cZ}\exp(f(\z))d\z, \quad N_2 = \int_{\cZ}\exp(f(\z)+g(\z))d\z.$$
First note that $\frac{Z_2}{Z_1}$ can be bounded as follows
\begin{align*}
    \frac{N_2}{N_1} &= \frac{1}{2}\int_{\cZ} \exp(g(\z)) \times \frac{\exp(f(\z))}{N_1}d\z = \Eover{\z\sim P_1}{\exp(g(\z))}.
\end{align*}
Since $\sup_{\z\in\cZ}|g(\z)|\leq B_g$, we have
\[
e^{-B_g}\leq \frac{N_2}{N_1} \leq e^{B_g}.
\]
We use this to bound the TV distance
\begin{align*}
    \text{TV}(P_1, P_2) & = \frac{1}{2}\int_{\cZ}\left|\frac{\exp(f(\z))}{N_1} - \frac{\exp(f(\z) + g(\z))}{N_2}\right| d\z\\
    & = \frac{1}{2}\int_{\cZ}\frac{\exp(f(\z))}{N_1}\left|1 - \frac{N_1}{N_2}\exp(g(\z)\right| d\z\\
    & \stackrel{(a)}{\leq} \frac{1}{2}\Eover{\z\sim P_1}{\max\left\lbrace\exp(2B_g)-1, 1-\exp(-2B_g)\right\rbrace}\\
    & \leq \frac{\exp(2B_g)-1}{2},
\end{align*}
where $(a)$ follows from our bound on $N_2/N_1$.
\end{proof}
\subsubsection{Proof of Theorem~\ref{thm:exp_alg_regret_bound}}
The proof of the Theorem relies on the following dual view of exponential weights update algorithm~\citep{maddison2014sampling, malmberg2012argmax}. In this view, the problem of sampling from distribution $P_t$ can be written as the following optimization problem
\begin{align*}
    \z_t = \argmin_{\z \in \cZ} \sum_{i=1}^{t-1}f_i(\z) - \frac{\sigma(\z)}{\eta},
\end{align*}
where $\sigma(\z)$ is a Gumbel process over $\cZ$~\citep{maddison2014sampling}. This dual view shows that Algorithm~\ref{alg:exp} is nothing but a Follow-The-Perturbed-Leader (FTPL) algorithm with Gumbel process as the perturbation. To bound its regret, we rely on the following regret bound of FTPL which holds for any distribution $P$ over $\cZ$~\citep[see Lemma 4 of][]{suggala2020online}
\begin{align}
\label{eqn:ftpl_exp_regret_bound}
    \sum_{t=1}^T\Eover{\z\sim P_t}{f_t(\z)} - \sum_{t=1}^T\Eover{\z\sim P}{f_t(\z)} \leq \sum_{t=1}^T \underbrace{\left(\Eover{\z\sim P_t}{f_t(\z)} - \Eover{\z\sim P_{t+1}}{f_t(\z)}\right)}_{T_1} + \underbrace{\frac{1}{\eta}\left(\Eover{\sigma, \z\sim P_1}{\sigma(\z)} -\Eover{\sigma, \z\sim P}{\sigma(\z)} \right)}_{T_2}
\end{align}
 $T_1$ can be bounded as follows
\begin{align*}
    \Eover{\z\sim P_t}{f_t(\z)} - \Eover{\z\sim P_{t+1}}{f_t(\z)} &\stackrel{(a)}{\leq} \text{TV}(P_t,P_{t+1})\|f_t\|_{\infty}\\
    &\stackrel{(b)}{\leq} B\times \text{TV}(P_t,P_{t+1})\\
    &\stackrel{(c)}{\leq} \frac{B}{2}\left(\exp(2\eta B)-1\right),
\end{align*}
where $\|f_t\|_{\infty}$ in $(a)$ is defined as $\sup_{\z\in\cZ}|f_t(\z)|$. Inequality $(b)$ in the above equation follows from our assumption that $f_t$ is bounded by $B$, and $(c)$ follows from Lemma~\ref{lem:exp_stability}.
$T_2$ can be bounded as follows
\begin{align*}
    \Eover{\sigma,\z\sim P_1}{\sigma(\z)} -\Eover{\sigma,\z\sim P}{\sigma(\z)}  \leq \Eover{\sigma}{\text{TV}(P_t,P_{t+1})\|\sigma\|_{\infty}}\leq \Eover{\sigma}{\|\sigma\|_{\infty}},
\end{align*}
where the last inequality follows from the fact that total variation distance between any two distributions is upper bounded by $1$.
Following \citet{maddison2014sampling}, the last term can be upper bounded by $|\log{\text{Vol}(\cZ)}|$. 
Plugging in the bounds of $T_1$ and $T_2$ into Equation~\eqref{eqn:ftpl_exp_regret_bound} gives us 
\begin{align*}
    \sum_{t=1}^T\Eover{\z\sim P_t}{f_t(\z)} - \sum_{t=1}^T\Eover{\z\sim P}{f_t(\z)} \leq \frac{|\log{\text{Vol}(\cZ)}|}{\eta} + \frac{BT}{2}\left(\exp(2\eta B)-1\right).
\end{align*}
For our choice of $\eta$, we can upper bound $\exp(2\eta B)-1$ by $4\eta B$. Plugging this into the above equation gives us the required bound on regret.
\subsection{Game Theory}
\label{sec:games}
Consider the following two-player zero-sum game
\[
\min_{\x\in \cX}\max_{\y \in \cY} f(\x,\y) 
\]
A pair $(\x^*,\y^*) \in \cX \times \cY$ is called a pure strategy Nash Equilibrium (NE) of the game, if the following holds
\[
\max_{\y \in \cY} f(\x^*, \y) \leq {f(\x^*,\y^*)} \leq \min_{\x \in \cX} {f(\x, \y^*)}.
\]
Intuitively, this says that there is no incentive for any player to change their strategy while the other player keeps hers unchanged. 
A pure strategy NE need not always exist though. What exists often is a mixed strategy NE~\citep{sion1958general}, which is defined as follows.
Let $P^*, Q^*$ be probability distributions over $\cX, \cY$. The pair $(P^*,Q^*)$ is called a mixed strategy NE of the above game if
\begin{align*}
    \sup_{\y \in \cY} \Eover{\x\sim P^*}{f(\x, \y)} \leq \Eover{\x\sim P^*}{\Eover{\y\sim Q^*}{f(\x,\y)}} \leq \inf_{\x \in \cX} \Eover{\y\sim Q^*}{f(\x, \y)}.
\end{align*}
Note that $(P^*,Q^*)$ can also be viewed as a pure strategy NE of the following linearized game\footnote{A linearized game is nothing but a game in the space of probability measures.}
\[
\inf_{P\in \Delta(\cX)}\sup_{Q \in \Delta(\cY)} \Eover{\x\sim P}{\Eover{\y\sim Q}{f(\x,\y)}}.
\]
Finally, $(P^*,Q^*)$ is called an $\epsilon$-approximate mixed NE of the game if
\begin{align*}
    \sup_{\y \in \cY} \Eover{\x\sim P^*}{f(\x, \y)} -\epsilon\leq \Eover{\x\sim P^*}{\Eover{\y\sim Q^*}{f(\x,\y)}} \leq \inf_{\x \in \cX} \Eover{\y\sim Q^*}{f(\x, \y)} + \epsilon.
\end{align*}
\subsection{Deriving $\robboostalg$}
\label{sec:deriving_mrboost}
In this work, we are interested in computing mixed strategy NE of the game in Equation~\eqref{eqn:robust_boosting_game}, which we present here for convenience
\begin{align*}
\max_{Q \in \Delta(\H)}\min_{(\x,y,y',\delta) \in S_{\text{aug}}} [h_Q(\x)]_y - [h_Q(\x)]_{y'}.
\end{align*}
This game can equivalently be written as follows
\begin{align*}
\max_{Q \in \Delta(\H)}\min_{P \in \Delta(S_{\text{aug}})} [h_Q(\x)]_y - [h_Q(\x)]_{y'}.
\end{align*}
This follows since the optimum of a linear program over a convex hull can always be attained at an extreme point. Using our definition of $\pairwiseloss{\cdot}$, we can rewrite the above game as
\begin{align}
\label{eqn:robust_boosting_game_linear}
\min_{Q \in \Delta(\H)}\max_{P \in \Delta(S_{\text{aug}})} \Eover{h\sim Q}{\Eover{(\x,y,y',\delta) \in P}{\pairwiseloss{h(\x+\delta), y, y'}}}.
\end{align}
A popular and widely used approach for finding mixed NE of this game is to rely on online learning algorithms~\citep{hazan2016introduction, cesa2006prediction}. In this approach, the minimization player and the maximization player play a repeated game against each other. Both the players rely on online learning algorithms to choose their actions in each round of the game, with the objective of minimizing their respective regret. In our work, we let the min player rely on Best Response (BR) and the max player rely on exponential weight update algorithm. 

We are now ready to describe our algorithm for computing a mixed strategy NE of Equation~\eqref{eqn:robust_boosting_game} (equivalently the pure strategy NE of the linearized game in Equation~\eqref{eqn:robust_boosting_game_linear}). Let $(h_t, P_t)$ be the iterates generated by the algorithm in $t$-th iteration. The maximization player chooses distribution $P_t$ over $S_{\text{aug}}$ using exponential weights update
\[
P_{t}(\x,y,y',\delta) \propto\exp\left(\eta\sum_{j=1}^{t-1}\pairwiseloss{h_j(\x+\delta). y,y'}\right).
\]
The minimization player chooses $h_t$ using BR, which involves computing a minimizer of the following objective
\begin{align*}
      h_t = \argmin_{h \in \H} \mathbb{E}_{(\x,y, y', \delta) \sim P_{t}}[\pairwiseloss{h(\x+\delta), y,y'}].
\end{align*}
In Section~\ref{sec:mrboost_convergence_rates_proofs}, we show that this algorithm converges to a mixed strategy NE of Equation~\eqref{eqn:robust_boosting_game}.

\section{Proofs of Section~\ref{sec:mrboost}}
\label{sec:mrboost_convergence_rates_proofs}
\subsection{Proof of Theorem~\ref{thm:convergence_rate_robboost}}
The proof of this Theorem relies on the observation that the maximization player is using exponential weights update algorithm to generate its actions and the minimization player is using Best Response (BR) to generate its actions. To simplify the notation, we let $L(Q,P)$ denote
\[
L(Q,P) = \Eover{h \sim Q}{\mathbb{E}_{(\x,y, y', \delta) \sim P}[\pairwiseloss{h(\x+\delta), y,y'}]}.
\]
Note that $L(Q,P)$ is linear in both its arguments. With a slight overload of notation, we let $L(h,P), L(h, (\x,y,y',\delta))$ denote
\[
L(h,P) = \mathbb{E}_{(\x,y, y', \delta) \sim P}[\pairwiseloss{h(\x+\delta), y,y'}].
\]
\[
L(h,(\x,y,y',\delta)) = \pairwiseloss{h(\x+\delta), y,y'}.
\]
\paragraph{Regret of min player.} In the $t$-th iteration, the min player observes the loss function $L(\cdot, P_t)$ and chooses its action using BR
\[
h_t \in \argmin_{h \in \H} L(h, P_t). 
\]
Since $L(h_t,P_t)\leq \min_{h\in \H} L(h, P_t)$, we have
\begin{align}
\label{eqn:min_player_regret}
\sum_{t=1}^T L(h_t,P_t) - \min_{h \in \H}\sum_{t=1}^TL(h, P_t) \leq 0.
\end{align}
\paragraph{Regret of max player.} In the $t$-th iteration, the min player chooses its action using exponential weights update algorithm and observes the loss function $L(h_t, \cdot)$. In particular, it plays action $P_t$, which is defined as follows
\[
P_{t}(\x,y,y',\delta) \propto\exp\left(\eta\sum_{j=1}^{t-1}L(h_t, (\x,y,y',\delta))\right). 
\]
Relying on the regret bound for exponential weights update algorithm derived in Theorem~\ref{thm:exp_alg_regret_bound}, and using the fact that $|L(h, (\x,y,y',\delta))|$ is bounded by $1$,  we get
\begin{align}
\label{eqn:max_player_regret}
    \max_{P \in \Delta(S_{\text{aug}})}\sum_{t=1}^TL(h_t, P)  - \sum_{t=1}^T L(h_t,P_t)  \leq 3\sqrt{T}\left(|\log{\text{Vol}(S_{\text{aug}})}|+1\right),
\end{align}
where $\text{Vol}(S_{\text{aug}}) = nK \text{Vol}(\gB(\epsilon))$.
Combining the regret bounds of the min and max players in Equations~\eqref{eqn:min_player_regret},~\eqref{eqn:max_player_regret}, we get
\begin{align}
\label{eqn:min_max_regret_bound}
    \max_{P \in \Delta(S_{\text{aug}})}\frac1T\sum_{t=1}^TL(h_t, P) - \min_{h \in \H}\frac1T\sum_{t=1}^TL(h, P_t) \leq 3\left(|\log{\text{Vol}(S_{\text{aug}})}|+1\right)T^{-1/2}.
\end{align}
Let $\xi(T) = 3\left(|\log{\text{Vol}(S_{\text{aug}})}|+1\right)T^{-1/2}$. We have the following from the above equation
\[
\max_{P \in \Delta(S_{\text{aug}})}L(Q(T),P) \leq \min_{Q\in \Delta(H)}\max_{P\in \Delta(S_{\text{aug}})} L(Q,P) + \xi(T).
\]
Rearranging the terms in the above equation and relying our definitions of $L(Q,P), \marginrob{\cdot}$, we get the following 
\begin{align*}
    \max_{Q\in \Delta(\H)} \marginrob{Q, S} \leq \marginrob{Q(T), S} + \xi(T).
\end{align*}
This proves the first part of the Theorem. The second part of the Theorem follows directly from the regret bound of the max player in Equation~\eqref{eqn:max_player_regret}.

\begin{algorithm}[t]
\caption{$\samplerrnd$}
\label{alg:sampler_rnd}
\begin{algorithmic}[1]
  \small
  \STATE \textbf{Input:} training data $S$, base models $\{\theta_j\}_{j=1}^t$.
  \STATE  $\widehat{S}_B \leftarrow \{\}$
  \STATE Uniformly sample a batch of points $\{(\x_b, y_b)\}_{b=1}^B$ from $S$.
  \FOR{$b = 1 \dots B$}
  \STATE uniformly sample $y'_b$ from $\cY\setminus\{y_b\}$, and compute $\delta_b$ as
  \begin{align*}
      &\delta_b \in \argmax_{\delta \in \gB(\epsilon)} \cemargin\left(\sum_{j=1}^{t}g_{\theta_j}(\x_b+\delta), y_b, y'_b\right)
  \end{align*}
  \STATE $\widehat{S}_B \leftarrow \widehat{S}_B \cup \{(\x_b, y_b, y'_b, \delta_b)\}.$ 
  \ENDFOR
  \STATE \textbf{Output:}  $\widehat{S}_B$
\end{algorithmic}
\end{algorithm}

\begin{algorithm}[t]
\caption{$\samplermax$}
\label{alg:sampler_max}
\begin{algorithmic}[1]
  \small
  \STATE \textbf{Input:} training data $S$, base models $\{\theta_j\}_{j=1}^t$.
  \STATE  $\widehat{S}_B \leftarrow \{\}$
  \STATE Uniformly sample a batch of points $\{(\x_b, y_b)\}_{b=1}^B$ from $S$.
  \FOR{$b = 1 \dots B$}
  \STATE  compute $(y'_b,\delta_b)$ as
  \begin{align*}
  &(y'_b,\delta_b) \in \argmax_{\delta \in \gB(\epsilon), y'\in \cY\setminus\{y_b\}} \cemargin\left(\sum_{j=1}^{t}g_{\theta_j}(\x_b+\delta), y_b, y'\right)
  \end{align*}
  \STATE $\widehat{S}_B \leftarrow \widehat{S}_B \cup \{(\x_b, y_b, y'_b, \delta_b)\}.$
  \ENDFOR
  \STATE \textbf{Output:}  $\widehat{S}_B$
\end{algorithmic}
\end{algorithm}

\section{$\robboostalgnn$}
\label{sec:mrboost_nn_details}

\subsection{Computationally Efficient Samplers}
\label{sec:samplers_efficient}
As previously mentioned, the sampling sub-routine in Algorithm~\ref{alg:sampler_exp} is the key bottleneck in our  Algorithm~\ref{alg:mrrob_nn}. So we design computationally efficient samplers which replace ``soft weight'' assignments in Algorithm~\ref{alg:sampler_exp} with ``hard weight'' assignments. Roughly speaking, these samplers first randomly sample a point $(\x,y,y')\in \Tilde{S}_{\text{aug}}$, and find the worst perturbations for it and use it to train the base classifier.\footnote{recall $\Tilde{S}_{\text{aug}} = \{(\x,y, y'): (\x,y)\in S, y' \in \cY\setminus\{y\}\}$} In Algorithm~\ref{alg:sampler_all} we described one of these samplers. Algorithms~\ref{alg:sampler_rnd},~\ref{alg:sampler_max} below present two other samplers. All these samplers differ in the way they perform hard weight assignment. 

\paragraph{$\samplerrnd$.} This is the most intuitive sampler. Here, we first uniformly sample $(\x,y)$ from $S$, and then  randomly pick a false label $y'\in \cY \setminus \{y\}$. Next, we generate a perturbation with the highest pairwise margin loss for the ensemble $\{\theta_j\}_{j=1}^t$. This involves solving the following optimization problem
\begin{align*}
\delta^* \in \argmax_{\delta \in \gB(\epsilon)} \cemargin\left(\sum_{j=1}^{t}g_{\theta_j}(\x+\delta), y, y'\right)
\end{align*}
\paragraph{$\samplermax$.} This sampler differs from $\samplerrnd$ in the way it chooses the label $y'$. Instead of randomly choosing $y'$ from $\cY \setminus \{y\}$, it picks the worst possible $y'$. That is, it picks a $y'$ which leads to the highest pairwise margin loss. This involves solving the following optimization problem
\begin{align*}
  (y'^*,\delta^*) \in \argmax_{\delta \in \gB(\epsilon), y'\in \cY\setminus\{y_b\}} \cemargin\left(\sum_{j=1}^{t}g_{\theta_j}(\x+\delta), y, y'\right)
  \end{align*}
\paragraph{$\samplerall$.}  This sampler smooths the ``max'' operator in $\samplermax$ by replacing it with the average of pairwise margin losses of all possible false labels. 
This involves solving the following optimization problem
\begin{align*}
  \delta^* \in \argmax_{\delta \in \gB(\epsilon)}\sum_{y'\in \cY\setminus\{y\}} \cemargin\left(\sum_{j=1}^{t}g_{\theta_j}(\x+\delta), y, y'\right)
 \end{align*}
  
We found that the last sampler ($\samplerall$) works best in our experiments (see related content in Section~\ref{sec:more_results}.). So we use this it by default, unless mentioned otherwise. 

\section{More Details about Experiments}
\label{sec:more_details}

\subsection{Efficient Implementation of MCE loss}
\label{sec:efficiency}
As mentioned in the main text, for training a single network, our technique has as little as $5\%$ additional computational overhead over baselines. 
For training an ensemble of $T$ networks, the runtime scales linearly with $T$. 
The following table shows the running time (in seconds) of one epoch of adversarial training under different settings.
These experiments were run on an NVIDIA A100 GPU.
\begin{center}
\begin{small}
\begin{sc}
\begin{tabular}{l|cc}
\toprule
Setting & \begin{tabular}{@{}c@{}}AT+CE \\ (baseline)\end{tabular}   & \begin{tabular}{@{}c@{}}AT+MCE \\ (ours)\end{tabular} \\
\midrule
\midrule
SVHN+Res18 & $99$s & $101$s \\
CIFAR10+Res18 & $71$s & $72$s \\
CIFAR100+Res18 & $70$s & $72$s \\
\bottomrule
\end{tabular}
\end{sc}
\end{small}
\end{center}

\subsection{Adversarial Attacks}

Given an input image $\x$, the target of an adversarial attack is to find an adversarial example $\x'$ within a local region of $\x$, such that it can mislead the classifier $g(\cdot)$ to make wrong decision making results.
The local region is usually defined to be an $\epsilon-$norm ball centred at $\x$.
In this work, we focus on $\ell_\infty$ cases, which means $\Vert\x'-\x\Vert_\infty\le\epsilon$.

One of the earliest attacks is Fast Gradient Sign Method (FGSM) \citep{goodfellow2014explaining}, which is generated in the following way:
$$
\x' = \x + \epsilon\cdot\textrm{sign}(\nabla_{\x}\ell(g(\x), y)),
$$
which is essentially one-step gradient ascent in the input space.
Another commonly used attack is Projected Gradient Descent (PGD) \citep{madry2017towards}, which perturbs the data for $K$ steps:
$$
\x'^{k+1} =  \Pi_\epsilon\left(\x^k +\alpha\cdot\textrm{sign}(\nabla_{\x^k}\ell(g(\x^k), y))\right),
$$
where $\alpha$ is the attack step size, and $\Pi_\epsilon$ denotes a projection to the $\ell_\infty$ norm ball.
The Auto Attack \citep{Croce2020ReliableEO} is a strong and reliable evaluation suite,
including a collection of three white-box attacks (APGD-CE \citep{Croce2020ReliableEO}, APGD-DLR \citep{Croce2020ReliableEO} and FAB \citep{Croce2020MinimallyDA}) and one black-box Square Attack \citep{Andriushchenko2020SquareAA}.

\subsection{Missing Details}
\label{sec:mce_variants}

In this section, we present more details about our experiments that are missing in the main paper.
We first explain how we combine the proposed MCE loss into existing defense algorithms. To simplify the notation, we define $\cemargina$ as follows
\[
\cemargina(g(\x), y) \coloneqq \left(\frac{1}{K-1}\sum_{y'\in \cY \setminus\{y\}}\cemargin(g(\rvx), y, y')\right).
\]
\textbf{TRADES.} 
TRADES \citep{zhang2019theoretically} proposes to use the following training objective:
$\ce(g(\rvx), y)+\lambda \cdot\max_{\Vert\x'-\x\Vert\le\epsilon} \left(  \KL{\left(\rvp(\rvx) \| \rvp\left({\rvx}^{\prime}\right)\right)}\right)$, where $\rvp=\sm(g)$.
We propose to incorporate MCE into this training framework as follows:
$$\cemargina(g(\x), y)+\lambda \cdot\max_{\Vert\x'-\x\Vert\le\epsilon} \left(  \KL{\left(\rvp(\rvx) \| \rvp\left({\rvx}^{\prime}\right)\right)}+
\KL{\left(\tilde\rvp(\rvx) \| \tilde\rvp\left({\rvx}^{\prime}\right)\right)}
\right),$$ where $\tilde\rvp=\sm(-g)$.
In our experiments, we set $\lambda=6$.

\textbf{MART.}
MART \citep{wang2019improving} uses the following objective:
${\bce}\left(g\left({\rvx}^{\prime}\right), y\right)+\lambda \cdot \KL{\left(\rvp(\rvx) \| \rvp\left({\rvx}^{\prime}\right)\right)} \cdot\left(1-\rvp(\rvx)_{y}\right)$ 
where $\x'$ is an adversarial example from attack standard cross entropy loss. We refer readers to the original paper \citep{wang2019improving} for more details on the BCE loss.
We propose to incorporate MCE into this training framework as follows:
$$\cemargina(g(\x), y)
+\lambda \cdot \left( \KL{\left(\rvp(\rvx) \| \rvp\left({\rvx}^{\prime}\right)\right)} \cdot\left(1-\rvp(\rvx)_{y}\right) 
+ \KL{\left(\tilde\rvp(\rvx) \| \tilde\rvp\left({\rvx}^{\prime}\right)\right)\cdot\left(1-\tilde\rvp(\rvx)_{y}\right) }
\right).$$
The notation of $\tilde\rvp$ is defined above.
In our experiments, we set $\lambda=1$.

\textbf{AWP.} 
AWP \citep{Wu2020AdversarialWP} proposes to optimize $\min_{\vtheta}\max_{\Vert\tilde\vtheta-\vtheta\Vert\le \gamma}\E_{\x, y}\left[ \max_{\Vert\x'-\x\Vert\le\epsilon} \ce(g_{\tilde\vtheta}(\x), y)\right]$, where $\gamma$ is an additional hyperparameter to constrain the perturbation on model weights $\tilde\vtheta$.
As before, we simply replace cross entropy loss in this objective with $\cemargina$.

\textbf{GAIR.}
GAIR \citep{Zhang2020GEOMETRYAWAREIA} develops a reweighting method with cross entropy loss to improve adversarial training, which is orthogonal to our contribution.
We simply replace the cross entropy loss in both training and attacks with $\cemargina$ to come up with GAIR+MCE method.
For the hyperparameters, we use the default settings as the original paper suggests.
We use the sigmoid-type decreasing function, set $\lambda=0$, set the $\lambda$ schedule to be the ``fixed" schedule, and set the GAIR beginning epoch to be the time of first learning rate decay.


\textbf{Training and Attack Details.} In boosting experiments, the number of parameters of the five ResNet-18 ensemble is $55869810$, while the deeper ResNet-158 model has $58156618$ learnable parameters.
We use 100 epochs for each boosting iteration (same as we do in Section~\ref{sec:single_model_experiment}) and run for five iterations.
The output logit of the ensemble model is defined as the average of logits from different base classifiers.
We test the robustness of all boosting methods with PGD-20 attack.

\makeatletter 
  \newcommand\figcaption{\def\@captype{figure}\caption} 
  \newcommand\tabcaption{\def\@captype{table}\caption} 
\makeatother
\begin{figure}[t] 
\small
\centering
  \begin{minipage}[b]{0.6\textwidth} 
  \begin{minipage}[b]{0.5\textwidth} 
    \centering 
    \includegraphics[width=0.98\textwidth]{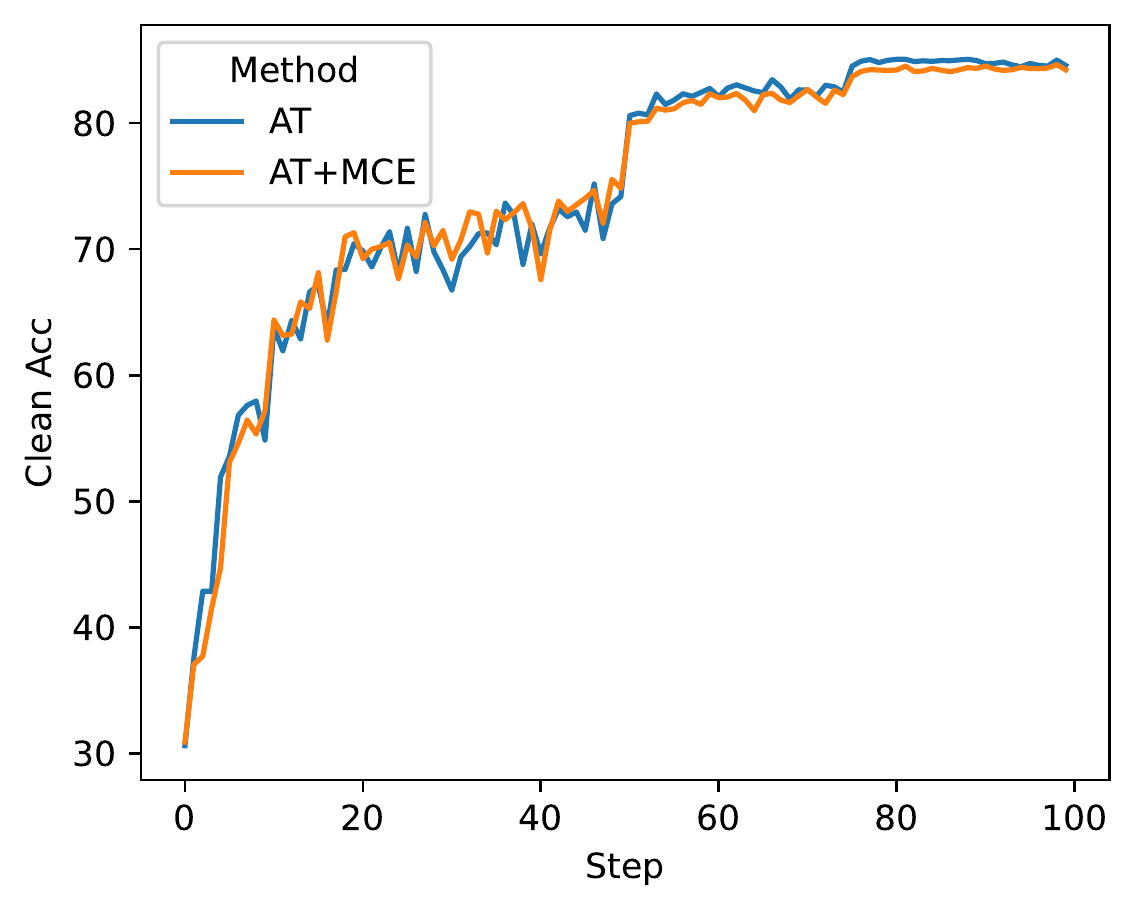}
  \end{minipage}%
  \hspace{-0.2cm}  
  \begin{minipage}[b]{0.5\textwidth} 
    \centering 
    \includegraphics[width=0.98\textwidth]{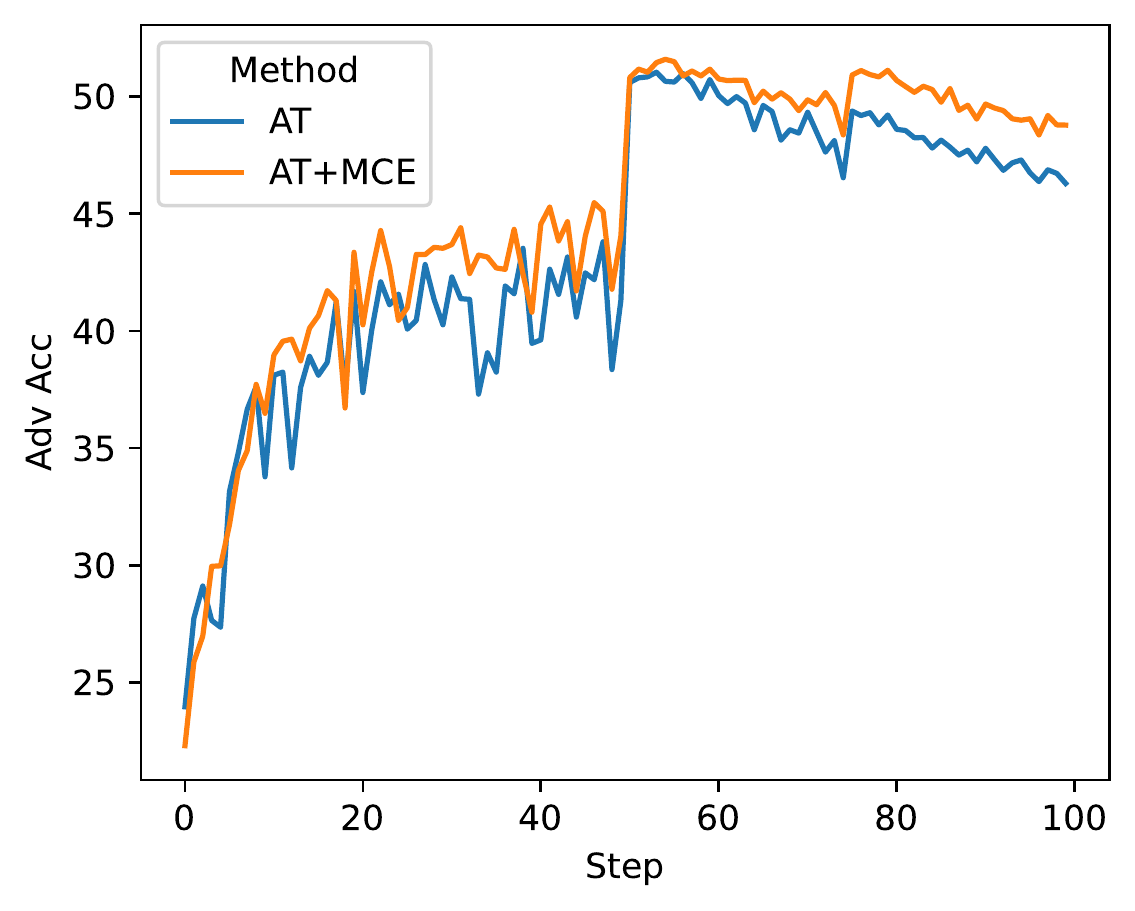}
  \end{minipage}
   \vskip -0.4cm
  \label{fig:at_atmce}
  \caption{ResNet-18 robustness of AT and AT+MCE.}
  \end{minipage}
  \hspace{-0.1cm}
  \begin{minipage}[b]{0.38\textwidth} 
    \centering
    \begin{sc}
    \begin{tabular}{c|cc}
    \toprule
    Method & Clean & Adv \\
    \midrule
    \midrule
    \samplerall & $82.06$ & $50.95$  \\
    \samplerrnd & $82.24$ & $50.77$  \\
    \samplermax & $68.86$ & $44.54$ \\
        AT($2$) & $81.62$ & $50.87$ \\
    \bottomrule
    \end{tabular}
    \end{sc} 
    \vspace{0.6cm}
    \tabcaption{Ablation of the proposed methods. 
    } 
    \label{tab:loss_ablation} 
  \end{minipage} 
\end{figure}

\subsection{More Results}
\label{sec:more_results}

In Table~\ref{tab:loss_ablation}, we perform an ablation study between different samplers. For this experiment, we use CIFAR-10 dataset.
We report the results of running  Algorithm~\ref{alg:mrrob_nn} with different samplers on ResNet-18 for 1 boosting iteration (where the base classifier is trained for 100 epochs).
From the table we could see that \samplerrnd\ is close to \samplerall\ but with slightly lower robust accuracy and higher clean accuracy.
We also compare with ``AT(2)'', which means we double the training loss of AT, to indicate that the improvement of \samplerall\ is not a result of a different loss scale.
Given these results, we use $\samplerall$ by default in our experiments presented in the main paper.

To show the difference between AT and AT+MCE, we also plot the performance curve of ResNet-18 for 100 epochs training in Figure~\ref{fig:at_atmce}.
As an evidence that our method is not fake defense, we try to attack the ResNet-18 model from AT+MCE  with an adaptively designed attack \citep{tramer2020adaptive}.
To be concrete, we replace the cross entropy loss computation in PGD-20 attack with MCE loss. 
The adversarial accuracy from this modified attack almost remains the same ($50.95\to 51.19$).
Considering our method has better adversarial accuracy even under AutoAttack, we believe our defense is not a result of gradient masking / obfuscated gradient.

We also present a more detailed version of boosting results in Table \ref{tab:boosting_ablation}.
Because the $\abernehty$ algorithm of \citet{abernethy2021multiclass} uses the whole ensemble $g_{1:t} = \sum_{j=1}^t g_j / t\ $ rather than only the new model $g_t$ to calculate the loss for new model optimization (see Appendix~\ref{sec:abernethy}),
we also try this version of $\robboostalgnn$ and name it $\robboostalgnn+\whole$ in this section.
We then use $\abernehty+\textsc{Ind}$ to denote the version of $\abernehty$ which only relies on $g_t$ to calculate the loss of the new model optimization, where ``ind'' stands for individual update.
To summarize, the improvement of our methodology is consistent for different settings.

\begin{table*}[t]
\caption{Boosting experiments with ResNet-18 being the base classifier.
}
\label{tab:boosting_ablation}
\centering
\scriptsize
\begin{sc}
\begin{tabular}{l|cccccccccc}
\toprule
\multirow{2}{*}{ Method } & \multicolumn{2}{c}{ Iteration 1 } & \multicolumn{2}{c}{ Iteration 2 }  & \multicolumn{2}{c}{ Iteration 3 }  & \multicolumn{2}{c}{ Iteration 4 }  & \multicolumn{2}{c}{ Iteration 5 } \\
&  Clean & Adv&  Clean & Adv&  Clean & Adv&  Clean & Adv&  Clean & Adv \\
\midrule
\midrule
Wider model & $82.61$ & $ 51.73$ & --- & --- & --- & --- & --- & --- & --- & ---  \\
Deeper model &  $82.67$ & $52.32$  & --- & --- & --- & --- & --- & --- & --- & ---  \\
\midrule
$\abernehty + \whole+\randinit$ & $82.00$ & $51.05$ & $84.58$ & $49.95$ & $83.87$ & $51.66$ & $82.56$ & $52.72$ & $81.44$ & $52.92$ \\
$\robboostalgnn + \whole + \randinit$ & $81.25$ & $51.80$ & $85.02$ & $52.29$ & $84.50$ & $53.11$ & $84.31$ & $53.58$ & $83.61$ & $54.01$ \\
\midrule
$\abernehty + \whole+\perinit$ & $82.18$ & $50.97$ & $85.60$ & $50.13$ & $84.59$ & $51.77$ & $84.21$ & $52.79$ & $82.78$ & $53.28$ \\
$\robboostalgnn + \whole + \perinit$ & $81.15$ & $51.68$ & $85.73$ & $52.59$ & $85.49$ & $53.12$ &$ 85.10$ & $53.72$ &$ 84.62$ & $54.00$ \\
\midrule
$\abernehty + \textsc{Ind} + \randinit$ & $82.16$ & $50.95$ & $85.13$ & $50.57$ & $85.55$ & $51.58$ & $85.75$ & $51.88$ & $85.92$ & $51.99$ \\
$\robboostalgnn + \textsc{Ind} + \randinit$ & $81.04$ & $51.83$ & $84.61$ & $52.68$ & $84.93$ & $53.51$ & $85.01$ & $53.95$ & $85.35$ & $54.13$ \\
\midrule
$\abernehty + \textsc{Ind} + \perinit$ & $82.12$ & $51.04$ & $85.73$ & $50.91$ & $85.98$ & $51.89$ & $85.97$ & $52.31$ & $85.81$ & $52.52$\\
$\robboostalgnn  + \textsc{Ind} + \perinit$ & $81.34$ & $51.92$ & $84.97$ & $52.97$ & $85.28$ & $53.62$ & $85.99$ & $54.26$ & $86.16$ & $54.42$ \\
\bottomrule
\end{tabular}
\end{sc}
\end{table*}

\subsection{Boosting algorithm of \citet{abernethy2021multiclass}}
\label{sec:abernethy}

\citet{abernethy2021multiclass} proposed a greedy stagewise boosting algorithm for building robust ensembles.  
Consider the following set of score-based base classifiers: $\G = \{g_{\theta}:\theta \in \R^D\},$ where $g_{\theta}:\cX \to \R^K$ is a neural network parameterized by $\theta$.  \citet{abernethy2021multiclass} construct ensembles of the form 
$g_{1:T}(\x) = \sum_{t=1}^T g_{\theta_t}(\x) / T$.\footnote{\citet{abernethy2021multiclass} also learn the weights of each component in the ensemble. In our experiments, we give equal weights to all the base classifiers in the ensemble.}
The authors rely on a greedy stagewise algorithm to build this ensemble. 
In the $t$-th stage of this algorithm, the method picks a base classifier $g_{\theta_t}$
via the following greedy procedure
\begin{align*}
    \theta_t = \argmin_{\theta\in \R^D} \frac{1}{n}\sum_{i=1}^n \max_{\delta \in \gB(\epsilon)} \ce\left(g_{1:t-1}(\x_i+\delta) + \frac{1}{t}g_{\theta}(\x_i+\delta), y_i\right).
\end{align*}
The authors solve this objective using AT~\citep{madry2017towards}. 
That is, for each step of gradient descent on $w, \theta$, we solve the inner optimization problem via gradient ascent on the $\delta$'s, with projections on to $\gB(\epsilon)$ at each step. 
Note that, at the beginning of $t$-th stage of the algorithm, $\theta_t$ is initialized randomly by \citet{abernethy2021multiclass}. 
In our experiments, we noticed that using persistent initialization (\emph{i.e.,} initializing $\theta_t$ to $\theta_{t-1}$) leads to more robust ensembles.

\subsection{Circumventing the Defense of \citet{Pinot2020RandomizationMH}}
\label{sec:pinot_defense_break}

\begin{table}[h]
\setlength{\tabcolsep}{4mm}
\caption{Demonstration of the failure of \citet{Pinot2020RandomizationMH} on CIFAR-10}
\label{tab:fake_defense}
\begin{center}
\begin{small}
\begin{tabular}{l  c}
\toprule
Testing scenario & Accuracy (\%) \\
\midrule
Clean data & $79.08$\\
Adversarial examples of $h_1$  & $54.60$ \\
\citet{Pinot2020RandomizationMH}'s ensemble attack  & $60.72$ \\
Our adaptive ensemble attack & $37.37$ \\
\bottomrule
\end{tabular}
\end{small}
\end{center}
\end{table}

\citet{Pinot2020RandomizationMH} is one of the first works to develop boosting inspired algorithms for building robust ensembles.
In this work, the authors build an ensemble by combining two base classifiers. The first one of these base classifiers is trained using PGD adversarial training (\emph{i.e.,} AT~\citep{madry2017towards}). 
The second base classifier is trained using standard empirical risk minimization on the adversarial dataset of the first base classifier. 
That is, the authors  train the second base classifier to be robust \emph{only} to the adversarial perturbations of the first base classifier. 
To be precise, after adversarially training the first neural network $g_1$, \citet{Pinot2020RandomizationMH} propose to additionally train a second model $g_2$ using standard training with the adversarial dataset of the first model $g_1$.
These two models are then {randomly} combined during evaluation, \ie, for each test input, the algorithm selects one classifier at random and then outputs the corresponding predicted class.
The authors claim that the proposed algorithm significantly boosts the performance of a single adversarially trained classifier.

In this work, we disprove this claim and show that the proposed defense can be circumvented using better adversarial attacks.
Our key insight here is that there is a mismatch between the attack techniques used during training and inference phases of \citet{Pinot2020RandomizationMH} (such phenomenon is referred to as ``Obfuscated Gradient'' phenomenon in \citet{Athalye2018ObfuscatedGG}\footnote{We refer the reader to Section 5.3 (``Stochastic Gradients'') of \citet{Athalye2018ObfuscatedGG} for more discussion on the failure of defenses that rely on randomization.}).
This mismatch often leads to a false sense of robustness.
Therefore, to properly evaluate the robustness of any defense technique, one should try to design adaptive attacks \citep{tramer2020adaptive} which take the algorithmic details of the defense into consideration.

Let $g_1, g_2$ be the base neural networks learned by the technique of \citet{Pinot2020RandomizationMH}. Let $w,(1-w)$ be the weights of these classifiers, where $w\in [0,1].$  Given a test input, the authors first sample a base classifier according to these weights, and output the predicted class of the chosen model. So the loss of this ensemble at a point $(\x, y)$ is given by
\begin{align}
\label{eqn:pinot_training}
w\ell_{0-1}(g_1(\x), y) + (1-w) \ell_{0-1}(g_2(\x), y).
\end{align}
To generate adversarial perturbation at a point $(\x,y)$ for this ensemble, the authors solve the following optimization problem using PGD
\[
\max_{\delta \in \gB(\epsilon)} \ce(wg_1(\x+\delta) + (1-w)g_2(\x+\delta), y).
\]
Note that during the attack, the aggregation of the base classifiers is performed at the logit level. Whereas, during training the aggregation is performed at the level of predictions (Equation~\eqref{eqn:pinot_training}). So, we design a slightly different attack which resolves this mismatch. In this attack the aggregation is performed at the probability level. This involves solving the following problem
\[
\max_{\delta \in \gB(\epsilon)} -\log\left(\frac{[w\rvp_1(\x+\delta)+(1-w)\rvp_2(\x+\delta)]_y}{\sum_{y'} [w\rvp_1(\x+\delta)+(1-w)\rvp_2(\x+\delta)]_{y'}}\right),
\]
where $\rvp_1(\x+\delta) = \sm(g_1(\x+\delta)), \rvp_2(\x+\delta) = \sm(g_2(\x+\delta)).$

Table \ref{tab:fake_defense} presents the performance of the ensembling technique of \citet{Pinot2020RandomizationMH} on various attacks. For this experiment, we use CIFAR-10 dataset and use a nine-layer convolutional neural network (as in \citet{Wang2019OnTC}) as the base classifier. It can be seen that the performance of the ensemble takes a hit when we use our adaptive attack to generate perturbations. In particular, there is a $23\%$ drop in robust accuracy when we switch from the attack of \citet{Pinot2020RandomizationMH} to our attack.  For comparison, performing standard AT on this model architecture results in a model with $80.29\%$ clean accuracy and $45.08\%$ PGD20 accuracy.

\end{document}